\documentclass[10pt,journal,compsoc]{IEEEtran}
\usepackage{setspace}

\usepackage{url}

\usepackage{mathtools}
\usepackage{amsmath,amsopn,amssymb}
\usepackage{amsthm} 
\usepackage{graphicx,xspace,color,soul}
\usepackage{epsfig,subfigure}
\usepackage{longtable,multirow}
\usepackage{array,float}
\usepackage{algorithmicx}
\usepackage{algorithm}
\usepackage{bm,epstopdf}
\usepackage{tabularx}
\usepackage{breqn}
\usepackage{hyperref}
\usepackage{changepage}

\newtheorem{theorem}{Theorem}
\newtheorem{definition}{Definition}
\newtheorem{lemma}{Lemma}

\newcolumntype{Y}{>{\centering\arraybackslash}X}

\def\eg{\textit{e.g.}}
\def\ie{\textit{i.e.}}
\def\etal{\textit{et al.~}}
\def\0{{\mathbf 0}}
\def\1{{\mathbf 1}}

\def\c{{\mathbf c}}
\def\f{{\mathbf f}}

\def\u{{\mathbf u}}
\def\v{{\mathbf v}}
\def\z{{\mathbf z}}

\def\A{{\mathbf A}}
\def\B{{\mathbf B}}

\def\D{\mathbf{D}}

\def\I{{\mathbf I}}
\def\L{{\mathbf L}}
\def\M{{\mathbf M}}

\def\S{{\mathbf S}}

\def\V{{\mathbf V}}
\def\W{{\mathbf W}}

\def\bLambda{{\boldsymbol \Lambda}}

\def\cE{{\mathcal E}}
\def\cG{{\mathcal G}}

\def\cN{{\mathcal N}}
\def\cS{{\mathcal S}}

\def\cU{{\mathcal U}}

\begin{document}

\title{Signed Graph Metric Learning via \\ Gershgorin Disc Perfect Alignment}

\author{Cheng Yang,~\IEEEmembership{Member,~IEEE},
Gene Cheung,~\IEEEmembership{Fellow,~IEEE},
Wei Hu,~\IEEEmembership{Senior Member,~IEEE}

\begin{small}
\thanks{Cheng Yang is with Shanghai Jiao Tong University, Shanghai 200240, China (e-mail: cheng.yang@ieee.org).}
\thanks{Gene Cheung is with York University, Toronto M3J 1P3, ON, Canada (e-mail: genec@yorku.ca).}
\thanks{Wei Hu is with Peking University, Beijing 100871, China (e-mail: forhuwei@pku.edu.cn).}
\thanks{Corresponding authors: Gene Cheung and Wei Hu.}
\thanks{Cheng Yang acknowledges the support of the China Postdoctoral Science Foundation grant No.2020TQ0194.}
\thanks{Gene Cheung acknowledges the support of the NSERC grants RGPIN-2019-06271, RGPAS-2019-00110.}
\thanks{Wei Hu acknowledges the support of National Natural Science Foundation of China (61972009) and Beijing Natural Science Foundation (4194080).}
\thanks{This work was done when Cheng Yang was with York University, Toronto M3J 1P3, ON, Canada.}
\end{small}
}

\IEEEtitleabstractindextext{%
\begin{abstract}
Given a convex and differentiable objective $Q(\M)$ for a real symmetric matrix $\M$ in the positive definite (PD) cone---used to compute Mahalanobis distances---we propose a fast general metric learning framework that is entirely projection-free. 
We first assume that $\M$ resides in a space $\cS$ of generalized graph Laplacian matrices corresponding to balanced signed graphs.
$\M \in \cS$ that is also PD is called a graph metric matrix.
Unlike low-rank metric matrices common in the literature, $\cS$ includes the important diagonal-only matrices as a special case. 
The key theorem to circumvent full eigen-decomposition and enable fast metric matrix optimization is Gershgorin disc perfect alignment (GDPA): 
given $\M \in \cS$ and diagonal matrix $\S$, where $S_{ii} = 1/v_i$ and $\v$ is the first eigenvector of $\M$, we prove that Gershgorin disc left-ends of similarity transform $\B = \S \M \S^{-1}$ are perfectly aligned at the smallest eigenvalue $\lambda_{\min}$. 
Using this theorem, we replace the PD cone constraint in the metric learning problem with tightest possible linear constraints per iteration, so that the alternating optimization of the diagonal / off-diagonal terms in $\M$ can be solved efficiently as linear programs via the Frank-Wolfe method.
We update $\v$ using Locally Optimal Block Preconditioned Conjugate Gradient (LOBPCG) with warm start as entries in $\M$ are optimized successively.
Experiments show that our graph metric optimization is significantly faster than cone-projection schemes, and produces competitive binary classification performance. 
\end{abstract}

\begin{IEEEkeywords}
Graph signal processing, metric learning, Gershgorin circle theorem, convex optimization
\end{IEEEkeywords}}

\maketitle
\thispagestyle{empty}

\IEEEpeerreviewmaketitle

\section{Introduction}
\label{sec:intro}
\IEEEPARstart{T}{he} notion of \textit{feature distance} $\delta_{ij}$ between two data samples $i$ and $j$, associated with respective feature vectors, $\f_i, \f_j \in \mathbb{R}^K$, is vital for many machine learning applications such as classification \cite{weinNNberger09LMNN}. 
Feature distance is traditionally computed as the \textit{Mahalanobis distance} \cite{mahalanobis1936},
$\delta_{ij} = (\mathbf{f}_i - \mathbf{f}_j)^{\top} \mathbf{M} (\mathbf{f}_i - \mathbf{f}_j)$, where $\M \in \mathbb{R}^{K \times K}$ is a \textit{metric matrix} assumed to be positive definite (PD)\footnote{Recent metric learning methods alternatively assumed $\M$ to be positive semi-definite (PSD), \ie, $\M \succeq 0$. Our methodology can handle both cases by appropriately setting a parameter $\rho \geq 0$ to be discussed.}, \ie, $\M \succ 0$ \cite{fsp2014}.
How to determine the best $\M$ given an objective function $Q(\M)$---\ie, $\min_{\M \succ 0} Q(\M)$---is the \textit{metric learning} problem.
We study this basic optimization problem in this paper. 

There is extensive prior work on the \textit{modeling} \cite{schultz04nips,globerson2006metric,qi09icml,zadeh16GMML}, \textit{optimization} \cite{Parikh31,GoluVanl96,yang20,hu2020feature}, and \textit{joint modeling / optimization} \cite{weinNNberger09LMNN, LSML, ericdml, lim13icml,liu15aaai,mu16aaai} of metric learning.
Modeling means new proposed definitions of objective $Q(\M)$, and optimization means new algorithms that solve $\min_{\M \succ 0} Q(\M)$ given $Q(\M)$.
A fundamental challenge in optimization of metric learning is to satisfy the PD cone constraint $\M \succ 0$ when minimizing a convex objective $Q(\M)$ in an efficient manner.
One na\"{i}ve approach is to first decompose $\M$ into $\M= \mathbf{G} \mathbf{G}^{\top}$ via \textit{Cholesky factorization} \cite{GoluVanl96}, where $\mathbf{G}$ is a lower-triangular matrix, and optimize $Q'(\mathbf{G}) = Q(\mathbf{G} \mathbf{G}^{\top})$ directly. 
However, doing so may mean a non-convex objective $Q'(\mathbf{G})$ with respect to variable $\mathbf{G}$, resulting in bad local minimums during non-convex optimization. 

Instead, one conventional and popular approach is alternating gradient-descent / projection (\eg, \textit{proximal gradient} (PG) \cite{Parikh31}), where a descent step $\alpha$ from current solution $\M^t$ at iteration $t$ in the direction of negative gradient $-\nabla Q(\M^t)$ is followed by a projection $\text{Proj}()$ back to the PD cone, \ie, $\M^{t+1} := \text{Proj} \left( \M^t - \alpha \nabla Q(\M^t) \right)$.
However, projection $\text{Proj}()$ requires eigen-decomposition of $\M^t$ and hard-thresholding of its eigenvalues per iteration, which has complexity $\mathcal{O}(K^3)$ and thus is expensive.

To avoid eigen-decomposition, recent methods consider alternative search spaces of matrices such as sparse or low-rank matrices to ease optimization \cite{qi09icml,lim13icml,liu15aaai,mu16aaai,zhang17aaai}. 
While efficient, the assumed search spaces are often overly restricted and degrade the quality of sought metric matrix $\M$. 
For example, low-rank methods assume reducibility of the $K$ available features to a lower dimension, and hence exclude the simple yet important weighted feature metric case where $\M$ is diagonal \cite{yang2018apsipa}, \ie, $(\mathbf{f}_i - \mathbf{f}_j)^{\top} \mathbf{M} (\mathbf{f}_i - \mathbf{f}_j) = \sum_{k} M_{kk} (f_i^k - f_j^k)^2$, $M_{kk} > 0, \forall k$. 

In this paper, we propose a fast, general metric learning framework, capable of optimizing any convex and differentiable objective $Q(\M)$, that entirely circumvents eigen-decomposition-based projection on the PD cone.
Compared to low-rank methods \cite{mu16aaai,liu15aaai}, our framework is more inclusive and includes diagonal metric matrices as a special case. 
Specifically, we first define a search space $\cS$ of generalized graph Laplacian matrices \cite{biyikoglu2005nodal}, each corresponding to a balanced\footnote{Balance for a feature graph means that if a feature $i$ is positively correlated with feature $j$, then feature $k$ positively correlated with $i$ cannot be negatively correlated with $j$. See Section\;\ref{subsec:GDA_Signed} for details.} signed graph.
If in addition $\M \succ 0$, then $\M$ is a \textit{graph metric matrix}.
In essence, an underlying graph $\cG$ corresponding to $\M \in \cS$ contains: i) edge weights reflecting pairwise (anti-)correlations among the $K$ features, and ii) self-loops designating relative importance among the features. 
Our proposed optimization enables fast searches within space $\cS$.

Our theoretical foundation is a new linear algebraic theorem called \textit{Gershgorin disc perfect alignment} (GDPA): for any matrix $\M \in \cS$, Gershgorin disc left-ends of similarity transform $\B = \S \M \S^{-1}$, where $\S$ is a diagonal matrix with $S_{ii} = 1 / v_i$ and $\v$ is the first eigenvector of $\M$, can be perfectly aligned at the smallest eigenvalue $\lambda_{\min}$.
Leveraging GDPA for fast metric optimization, we replace the PD cone constraint with a set of $K$ \textit{tightest possible}\footnote{By ``tightest possible", we mean that the lower bound $\lambda^-_{\min}(\B)$ of the smallest eigenvalue $\lambda_{\min}(\B)$---smallest Gershgorin disc left-end of matrix $\B$---and $\lambda_{\min}(\B)$ are the same. See Section\;\ref{sec:graph} for details.} linear constraints per iteration as follows:  
i) compute scalars $S_{ii} = 1/v_i$ from first eigenvector $\v$ of previous solution $\M^{t}$, ii) write $K$ linear constraints for $K$ rows of the next solution $\M^{t+1}$ using computed scalars $S_{ii}$ to ensure PDness of $\M^{t+1}$ via the \textit{Gershgorin Circle Theorem} (GCT) \cite{gahc}.
Linear constraints mean that our proposed alternating optimization of the diagonal / off-diagonal terms in $\M^{t+1}$ can be solved speedily as \textit{linear programs} (LP) \cite{co1998} via the Frank-Wolfe method \cite{pmlr-v28-jaggi13}. 
A flow chart of our GDPA-based optimization framework is shown in Fig.\;\ref{fig:algorithm_chart}, where in each minimization $\min Q(\M)$, the PD cone constraint is replaced by linear constraints defined using scalar $\{S_{ii}\}$, resulting in significant speedup. 

\begin{figure}[t]
    \centering
    \vspace{-0.15in}
    \includegraphics[width=0.48\textwidth]{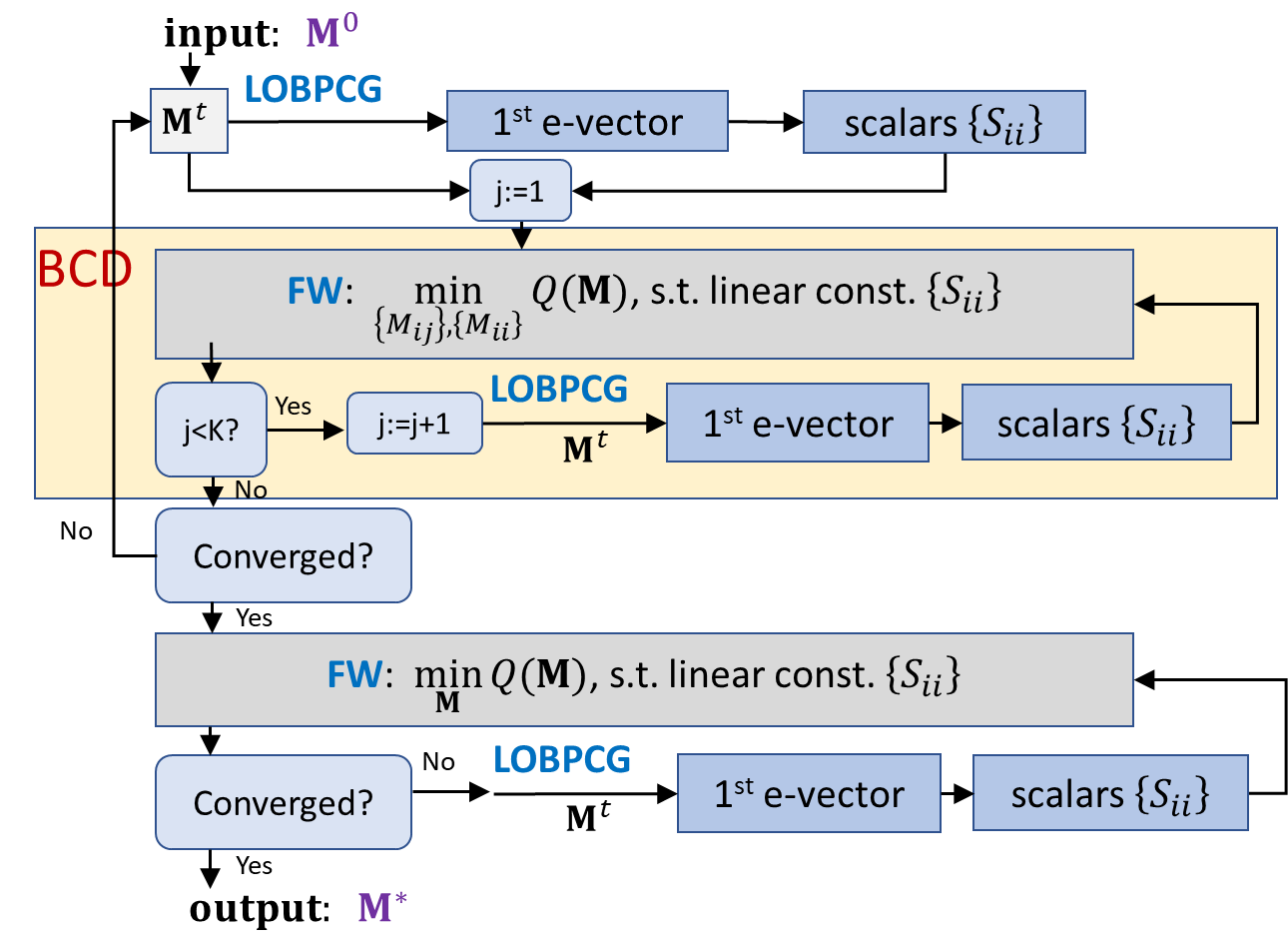}
    \vspace{-0.1in}
    \caption{GDPA-based optimization framework. BCD=Block coordinate descent. FW=Frank-Wolfe. LOBPCG=Locally Optimal Block Preconditioned Conjugate Gradient.}
    \label{fig:algorithm_chart}
\end{figure}

The bulk of the algorithm complexity resides in the repeated computation of the first eigenvectors $\v$ of $\M^{t}$.
We update $\v$ using \textit{Locally Optimal Block Preconditioned Conjugate Gradient} (LOBPCG) \cite{Knyazev01} with warm start as diagonal / off-diagonal terms are optimized successively.
Extensive experiments show that our graph metric optimization is significantly faster than cone-projection methods (up to 8x speedup for large $K$), and produces competitive binary classification performance.

The paper is organized as follows. 
We first review previous works in Section~\ref{sec:related}. 
We describe GDPA in Section~\ref{sec:graph}. 
Leveraging GDPA, we describe our metric optimization framework in Section~\ref{sec:learning} and \ref{sec:learning2}. 
Finally, experiments and conclusions are presented in Section~\ref{sec:results} and \ref{sec:conclude}, respectively.

\section{Related Work}
\label{sec:related}
We first divide existing methods into two main categories: linear and nonlinear distance metric learning. 

\subsection{Linear Distance Metric Learning}

These methods learn linear transformations to project samples into a new feature space. This paradigm is prevalent in the metric learning community, as many of the resulting transformations are tractable.  
Mahalanobis distance metric is one representative linear metric, which has been extensively studied under different assumptions. 
As illustrated in Fig.~\ref{fig:related}, we classify previous works in linear distance metric learning into three categories based on their key contributions: 
1) contributions in modeling; 2) contributions in optimization; and 3) contributions in joint modeling and optimization.  

The first class of related works focused on the design of novel metric learning objectives (modeling) while employing existing techniques and algorithms for optimization \cite{schultz04nips,globerson2006metric,qi09icml,zadeh16GMML}. 
Zadeh \etal \cite{zadeh16GMML} proposed a metric learning objective following intuitive geometric reasoning, resulting in an unconstrained, smooth, and strictly convex optimization problem that admits a closed-form solution.  
Globerson \etal \cite{globerson2006metric} proposed a convex optimization problem aiming to collapse all examples in the same class to a single point and push examples in other classes infinitely far away, and employed the projected gradient method to solve it. 
Qi \etal \cite{qi09icml} exploited the sparsity prior of distance metric learning, which is solved in a block coordinate descent fashion. 

\begin{figure}[t]
\centering
\includegraphics[width=0.45\textwidth]{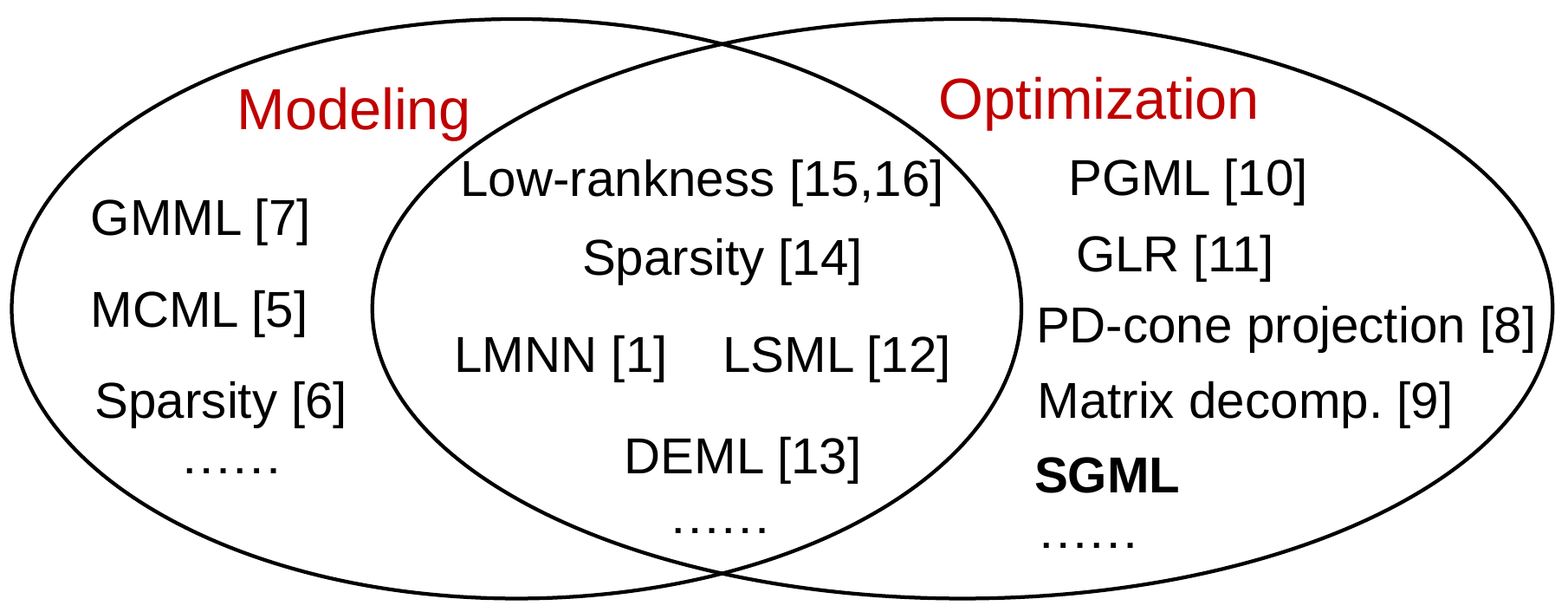}
\vspace{-0.12in}
\caption{Existing linear distance metric learning methods are classified based on contributions in modeling, optimization or joint modeling / optimization.}
\label{fig:related}
\end{figure}

The second class of works focused on new optimization methods for existing metric learning models. 
Classical optimization methods include PD-cone projection \cite{Parikh31} and matrix decomposition \cite{GoluVanl96}.
Further, Yang \etal \cite{yang20} proposed a projection-free metric learning framework based on positive graphs for a convex differentiable objective function of the metric matrix. 
Hu \etal \cite{hu2020feature} proposed a fast optimization method that mitigates
full eigen-decomposition of the distance metric specifically to minimize the \textit{graph Laplacian regularizer} (GLR) \cite{pang2017graph}.    
Our GDPA-based optimization framework belongs to this second class---we are agnostic to the choice of actual metric objective employed. 
Instead, we argue that \textit{our general optimization strategy can speedily minimize a broad class of metric objectives, requiring only that the chosen cost function $Q(\M)$ is convex and differentiable}.

Most related works fall into the third class that proposed new metric learning objectives accompanied by specialized optimization \cite{weinNNberger09LMNN, LSML, ericdml, lim13icml,liu15aaai,mu16aaai}. 
Many new models are based on assumptions of the solution space, such as low-rankness \cite{liu15aaai,mu16aaai} and sparsity \cite{lim13icml}.
For example, assuming that the desired metric $\M$ inherently lies in a lower dimension than the original $K$-dimensional feature space, Weinberger \etal \cite{weinNNberger09LMNN} proposed a convex model to maximize the margin by which the model correctly classifies labeled examples in the training set, and implemented a special-purpose solver for speedup. 
Liu \etal \cite{LSML} proposed to learn a Mahalanobis distance metric from supervision in the form of relative distance comparisons, along with a simple yet effective algorithm. 
Xing \etal \cite{ericdml} learned a distance metric that respects the relationships of given examples of similar and dissimilar pairs of points, and provided efficient, local-optima-free algorithms. 
In contrast, our framework makes no strong assumptions such as low-rankness or sparsity on the optimization variable $\M$ beyond its residence in the PD cone, and thus is more general.

\subsection{Nonlinear Distance Metric Learning}

Given possibly nonlinear relationship of data points, these methods learned nonlinear transformations to map samples into another feature space. 
While kernelized linear transformations can be adopted to address the nonlinear problem \cite{torresani2007large,mika1999fisher}, choosing a kernel is typically difficult and empirical, and often not flexible enough to capture the nonlinearity in the data. 
Given that deep learning is effective in modeling function nonlinearities, deep metric learning (DML) methods employed various deep neural network architectures to learn a set of hierarchical transformations for nonlinear mapping of data points \cite{lu2017deep}.
DML methods mainly include Siamese-networks based methods \cite{hadsell2006dimensionality,taigman2014deepface,hu2014discriminative,sun2014deep} and triplet-networks based methods \cite{wang2014learning, hoffer2015deep, schroff2015facenet,oh2016deep}. 
The objective functions were often designed for different specific tasks. 
DML has shown substantial benefits in wide applications of various visual understanding tasks such as image classification, visual search, visual tracking and so on. 
However, they were mostly trained in a supervised fashion, requiring a large amount of labeled data. 

In contrast, our metric optimization framework does not require bulk training data, and is suitable for any convex and differentiable objective, including proposals in \cite{globerson2006metric,weinNNberger09LMNN,ying2012distance,zadeh16GMML}.

\subsection{GDA-based Graph Sampling}

We studied \textit{Gershgorin disc alignment} (GDA) in the context of graph sampling in our previous work
\cite{bai19icassp,bai20tsp}. 
There are two key differences between our current work and \cite{bai19icassp,bai20tsp}.
First, we derive theorems to show \textit{perfect} alignment of Gershgorin disc left-ends at $\lambda_{\min}$ for defined classes of matrices, while disc alignment in \cite{bai19icassp,bai20tsp} was only approximate. 
Second, we apply disc alignment theory in an optimization framework for metric learning.
To differentiate from \cite{bai19icassp,bai20tsp}, we call our current work \textit{Gershgorin disc perfect alignment} (GDPA).

Our preliminary work on GDPA for metric learning \cite{yang20} assumes a restricted search space of Laplacian matrices for irreducible \textit{positive} graphs with positive node degrees.
Here, we generalize to a much larger space of Laplacian matrices for \textit{balanced signed} graphs.
While positive edges can encode positive correlations between features, negative edges in a signed graph can encode anti-correlations between features. 

We illustrate the usefulness of this generalization in metric learning using a 2-dimensional feature space, where the feature graph has (possibly negative) edge weight $w_{1,2}$ connecting features (nodes) 1 and 2, and there are self-loops $u_1 = u_2 = 2 |w_{1,2}|$ at the two nodes.
The resulting generalized graph Laplacian matrix\footnote{We discuss definitions of graph Laplacian matrices in Section\;\ref{subsec:laplacian}.} $\M$ is
\begin{align}
\M = \left[ \begin{array}{cc}
2 |w_{1,2}| + w_{1,2} & -w_{1,2} \\
-w_{1,2} & 2 |w_{1,2}| + w_{1,2}
\end{array} \right] .
\end{align}
This self-loop assignment ensures $\M$ is PSD regardless of the value of $w_{1,2}$ \cite{8296568}.
To achieve zero Mahalanobis distance between samples $i$ and $j$, clearly one possibility is when $\f_i = \f_j \in \mathbb{R}^2$, in which case $\delta_{ij}(\M) = (\f_i - \f_j)^{\top} \M (\f_i - \f_j) = 0$. If $w_{1,2} < 0$, there exists another possibility when $\f_i - \f_j = [\eta ~ -\eta]^{\top}$ for some $\eta \in \mathbb{R}$, in which case, 
\begin{align}
(\f_i - \f_j)^{\top} \M (\f_i - \f_j) &= \left[\eta ~-\eta \right]
\left[ \begin{array}{cc}
- w_{1,2} & - w_{1,2} \\
- w_{1,2} & - w_{1,2}
\end{array} \right] 
\left[ \begin{array}{c}
\eta \\
-\eta 
\end{array} \right] 
\nonumber \\
&= 0 .
\end{align}
Thus, a negative edge can encode anti-correlation in features and enable small Mahalanobis distance even when $\f_i \neq \f_j$. 
We will show in Section\;\ref{sec:results} that this generalization leads to noticeable performance gain when optimizing different objectives $Q(\M)$'s.

\section{Gershgorin Disc Perfect Alignment}
\label{sec:graph}
We first review basic definitions in \textit{graph signal processing} (GSP) \cite{8347162,8334407,ortega_2021,cheung_2021} that are necessary to understand our GDPA theory. 
We then describe GDPA for positive graphs and balanced signed graphs in order.

\subsection{Graph and Graph Laplacian Matrices}
\label{subsec:laplacian}

We consider an undirected graph $ \mathcal{G}=\{\cN,\cE, \cU\} $ containing a node set $\cN$ of cardinality $|\cN|=K$. 
Each inter-node edge $(i,j) \in \mathcal{E}$, $i\neq j$, has an associated weight $w_{ij} \in \mathbb{R}$ that reflects the degree of (dis)similarity or (anti-)correlation between nodes $i$ and $j$, depending on the sign of $w_{ij}$.
Each node $i$ may have a self-loop $(i) \in \cU$ with weight $u_i \in \mathbb{R}$. 

One can collect edge weights and self-loops into an \textit{adjacency matrix} $\W$, where $W_{ij} = w_{ij}, ~(i,j) \in \cE$, and $W_{ii} = u_i,~ (i) \in \cU$. 
We define a diagonal \textit{degree matrix} $\D$, where $D_{ii} = \sum_j w_{ij}$, that accounts for both inter-node edges and self-loops. 
The \textit{combinatorial graph Laplacian matrix} \cite{8347162} is defined as $ \L=\D-\W $. 
A \textit{generalized graph Laplacian matrix} \cite{biyikoglu2005nodal} accounts for self-loops in $\cG$ also and is defined as $\L_g = \D - \W + \text{diag}(\W) $, where $\text{diag}(\W)$ extracts the diagonal entries of $\W$.
Alternatively, we can write $\L_g = \D - \W_g$, where $\W_g = \W - \text{diag}(\W)$ contains only inter-node edge weights (diagonal terms are zeros).


\subsection{GDPA for Positive Graphs}

Consider first the simpler case of a \textit{positive graph}, where an irreducible\footnote{An irreducible graph $\cG$ means that there exists a path from any node in $\cG$ to any other node in $\cG$ \cite{irregraph}.} graph $\cG$ (no disconnected sub-graphs) has strictly positive edge weights and self-loops, \ie, $w_{ij} > 0, (i,j) \in \cE$ and $u_i > 0, (i) \in \cU$.
This means that $\W$ is non-negative, the diagonals in $\D$ are strictly positive, and the generalized graph Laplacian $\L_g$ is \textit{positive semi-deinite} (PSD) \cite{cheung2018graph}.
We discuss first GDPA for this case.

\subsubsection{Gershgorin Circle Theorem}
\label{subsubsec:GCT}

We first overview \textit{Gershgorin Circle Theorem} (GCT) \cite{gahc}. 
By GCT, each real eigenvalue $\lambda$ of a real symmetric matrix $\M$ resides in at least one \textit{Gershgorin disc} $\Psi_i$, corresponding to row $i$ of $\M$, with center $c_i = M_{ii}$ and radius $r_i = \sum_{j \,|\, j\neq i} |M_{ij}|$, \ie,
\begin{align}
\exists i ~~\mbox{s.t.}~~
c_i - r_i \leq \lambda \leq c_i + r_i.
\end{align}
Thus a sufficient (but not necessary) condition to guarantee $\M$ is PSD (\ie, smallest eigenvalue $\lambda_{\min} \geq 0$) is to ensure that the smallest Gershgorin disc left-end $\lambda_{\min}^-$---a lower bound for $\lambda_{\min}$---is non-negative, \ie,
\begin{align}
0 \leq \lambda_{\min}^- \triangleq \min_i c_i - r_i \leq \lambda_{\min} .
\end{align}

However, $\lambda_{\min}^-$ is often much smaller $\lambda_{\min}$, resulting in a loose lower bound.
As an illustration, consider the following example $3 \times 3$ PD matrix $\M$:
\begin{align}
\M = \left[ \begin{array}{ccc}
2 & -2 & -1 \\
-2 & 5 & -2 \\
-1 & -2 & 4
\end{array} \right] .
\label{eq:exM1}
\end{align}
Lower bound $\lambda_{\min}^- = \min (-1, 1, 1) = -1$, while the smallest eigenvalue for $\M$ is $\lambda_{\min} = 0.1078 > 0$.
See Fig.\;\ref{fig:gda_show} for an illustration of Gershgorin discs for this example.

\begin{figure}
\begin{center}
\includegraphics[width=3.5in]{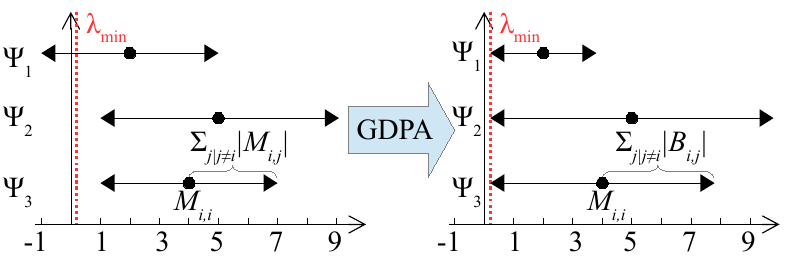}
\vspace{-0.2in}
\caption{Illustration of Gershgorin discs for matrix $\M$ in \eqref{eq:exM1} (left), and aligned discs for $\B = \S \M \S^{-1}$ (right).}
\label{fig:gda_show}
\end{center}
\end{figure}

\subsubsection{GDPA Analysis for Positive Graphs}

GDPA is a procedure to \textit{scale} the Gershgorin disc radii $r_i$ of matrix $\M$, so that all disc left-ends $c_i - r_i$ are perfectly aligned at $\lambda_{\min}(\M)$.
In other words, after GDPA the GCT lower bound $\lambda^-_{\min}(\M)$ for smallest eigenvalue $\lambda_{\min}(\M)$ is the \textit{tightest possible}.
Specifically, we perform a \textit{similarity transform} \cite{GoluVanl96} of $\M$ via matrix $\S$, \ie,
\begin{align}
\B = \S \M \S^{-1}
\label{eq:similarTrans}
\end{align}
where $\S = \text{diag}(s_1, \ldots, s_K)$ is chosen to be a diagonal \textit{scaling} matrix with scalars $s_1, \ldots, s_K$ along its diagonal, where $s_i > 0, \,\forall i$.
This means $\S$ is easily invertible and $\S^{-1}$ is well defined. 
$\B$ has the same eigenvalues as $\M$, and thus the smallest Gershgorin disc left-end for $\B$ is also a lower bound for $\M$'s smallest eigenvalue, \ie,

\begin{align}
\lambda_{\min}^-(\B) & \leq \lambda_{\min}(\B) = \lambda_{\min}(\M) \\
&= \min_i B_{ii} - \sum_{j \,|\, j\neq i} |B_{ij}| \\
&= \min_i M_{ii} - s_i \sum_{j \,|\, j \neq i} |M_{ij}| / s_j. 
\end{align}

We show that given a generalized graph Laplacian matrix $\M$ corresponding to an irreducible, positive graph $\cG$, there exist scalars $s_1, \ldots, s_K$ such that all Gershgorin disc left-ends of $\B = \S \M \S^{-1}$ are aligned exactly at $\lambda_{\min}(\M)$. 
We state this formally as a theorem.

\begin{theorem}
Let $\M$ be a generalized graph Laplacian matrix corresponding to an irreducible, positive graph $\cG$. 
Denote by $\v$ the first eigenvector of $\M$ corresponding to the smallest eigenvalue $\lambda_{\min}$.
Then, by computing scalars $s_i = 1/v_i, \forall i$, all Gershgorin disc left-ends of $\B = \S \M \S^{-1}$, $\S = \text{diag}(s_1, \ldots, s_K)$, are aligned at $\lambda_{\min}$, \ie, $B_{ii} - \sum_{j \,|\, j\neq i} |B_{ij}| = \lambda_{\min}, \forall i$.
\label{thm:GDA}
\end{theorem}

Continuing our earlier example, using $s_1 = 0.7511$, $ s_2 = 0.4886$ and $s_3=0.4440$, we see that  $\B = \S \M \S^{-1}$ for $\M$ in \eqref{eq:exM1} has all disc left-ends aligned at $\lambda_{\min} = 0.1078$.

\vspace{0.1in}
To prove Theorem 1, we first establish the following lemma. 

\begin{lemma}
\label{lemma:posEV}
There exists a first eigenvector $\v$ with strictly positive entries for a generalized graph Laplacian matrix $\M$ corresponding to an irreducible, positive graph $\cG$.
\end{lemma}

\begin{proof}
By definition, $\M$ is a generalized graph Laplacian $\M = \D - \W_g$ with positive inter-node edge weights in $\W_g$ and positive degrees in $\D$.
Let $\v$ be the first eigenvector of $\M$ corresponding to eigenvalue $\lambda_{\min} \geq 0$ ($\M$ is PSD), \ie, 
\begin{align}
\M \v &= \lambda_{\min} \v \nonumber \\
(\D - \W_g) \v &= (\lambda_{\min} \I) \v \nonumber \\
\v &= \D^{-1} (\W_g + \lambda_{\min} \I) \v \nonumber
\end{align}
where $\I$ is an identity matrix, and $\lambda_{\min} \geq 0$ since $\M$ is PSD.
Thus, matrix $\A = \D^{-1} (\W_g + \lambda_{\min} \I)$ has right eigenvector $\v$ corresponding to eigenvalue $1$. 
$\A$ contains only non-negative entries and has \textit{unit spectral radius}, \ie, $\rho(\A) = 1$ (see the proof in the Appendix).

Note also that $\A$ is an irreducible matrix (since $\W_g$ is irreducible). 
Thus, $\v$ is a strictly positive eigenvector corresponding to eigenvalue and spectral radius 1 of matrix $\A$ by the Perron-Frobenius Theorem \cite{ma2012}.
\end{proof}

We now prove Theorem 1 as follows.
\begin{proof}
Denote by $\v$ a strictly positive eigenvector corresponding to the smallest eigenvalue $\lambda_{\min}$ of $\M$. 
Define $\S = \mathrm{diag}(1/v_1, \ldots, 1/v_K)$.
Then,
\begin{align}
\S \M \S^{-1} \S \v = \lambda_{\min} \S \v 
\end{align}
where $\S \v = \1 = [1, \ldots, 1]^{\top}$.
Let $\B = \S \M \S^{-1}$.
Then,
\begin{align}
\B \1 = \lambda_{\min} \1 .
\label{eq:GCT_proof1}
\end{align}
\eqref{eq:GCT_proof1} means that
\vspace{-0.05in}
\begin{align}
B_{ii} + \sum_{j \,|\, j \neq i} B_{ij} &= \lambda_{\min}, ~~~ \forall i. \nonumber 
\end{align}
Note that the off-diagonal terms $B_{ij} = (v_i/v_j) M_{ij} \leq 0$, since: i) $\v$ is strictly positive,  and ii) off-diagonal terms of generalized graph Laplacian $\M$ for a positive graph satisfy $M_{i,j} \leq 0$. 
Thus,
\begin{align}
B_{ii} - \sum_{j \,|\, j \neq i} |B_{ij}| &= \lambda_{\min}, ~~~ \forall i.
\end{align}
Thus, defining $\S = \mathrm{diag}(1/v_1, \ldots, 1/v_K)$ means that $\B = \S \M \S^{-1}$ has all its Gershgorin disc left-ends aligned at $\lambda_{\min}$. 
\end{proof}

\subsection{GDPA for Balanced Signed Graphs}
\label{subsec:GDA_Signed}

We generalize our GDPA analysis to signed graphs, where edge weights and self-loops can be negative.
Central to our analysis is the concept of \textit{graph balance}.
We discuss graph balance and the related Cartwright-Harary Theorem (CHT) \cite{cht1956}, then present our GDPA analysis.

\subsubsection{Cartwright-Harary Theorem} 

\begin{figure}
\begin{center}
\includegraphics[width=3in]{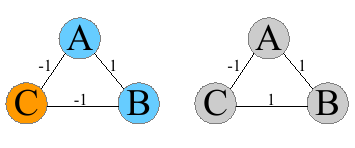}
\vspace{-0.2in}
\caption{Examples of a 3-node balanced signed graph (left) and an unbalanced signed graph (right).}
\label{fig:balanced_unbalanced}
\end{center}
\end{figure}

The concept of balance in a signed graph has been studied in many scientific disciplines, including psychology, social networks and data mining \cite{sgsn}.
We adopt the following definition of a balance graph for our analysis:
\begin{definition}
A signed graph $\cG$ is balanced iff $\cG$ does not contain any cycle with odd number of negative edges.
\end{definition}
For intuition, consider a graph $\cG$ with three nodes denoted by $A$, $B$ and $C$. 
Suppose that a positive/negative edge reflects pairwise friend/enemy relationship. 
An edge sign assignment of $(A,B)=1$ and $(B,C)=(C,A)=-1$---resulting in a cycle of two negative edges---means that $A$ and $B$ are friends, and that both $A$ and $B$ are enemies with $C$. 
See Fig.\;\ref{fig:balanced_unbalanced} for an illustration.
This graph is balanced; nodes can be grouped into \textit{two} clusters, $\{A, B\}$ and $\{C\}$, where nodes within a cluster are friends, and nodes across clusters are enemies. 

In contrast, an edge sign assignment of $(A,B)=(B,C)=1$ and $(C,A)=-1$---resulting in a cycle of one negative edge---means that both $A$ and $C$ are friends with $B$, but $A$ and $C$ are enemies.
This graph is not balanced; one cannot assign nodes to two distinct clusters with consistent signs as we did previously. 
We can generalize this example to the CHT \cite{cht1956} as follows.

\begin{theorem}
A graph $\cG$ is balanced iff its nodes $\cN$ can be partitioned into blue and red clusters, $\cN_b$ and $\cN_r$, such that a positive edge always connects two same-color nodes, and a negative edge always connects two opposite-color nodes.
\label{thm:CHT}
\end{theorem}
One interpretation of Theorem\;\ref{thm:CHT} is that if each cluster of nodes connected by positive edges in a balanced graph are merged into a single node, then the merged nodes, connected by negative edges only, form a \textit{bipartite graph}---a 2-colorable graph.  

There are two implications.
First, to determine if graph $\cG$ is balanced, instead of examining all cycles in $\cG$ to check if each contains an odd number of negative edges, one can check if nodes can be colored into blue and red with consistent edge signs as stated in Theorem \ref{thm:CHT}. 
Second, we can use CHT to prove that GDPA is possible for a Laplacian matrix corresponding to an irreducible, balanced signed graph.
We describe this next.

\subsubsection{GDPA Analysis for Signed Graphs}

\begin{figure}
\begin{center}
\includegraphics[width=3.4in]{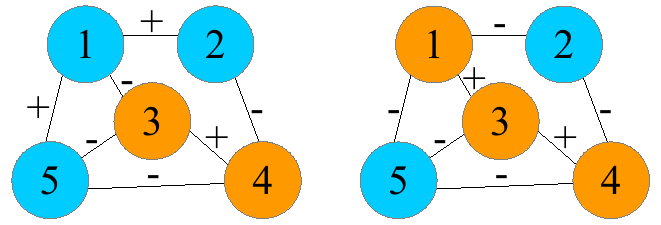}
\vspace{-0.1in}
\caption{Example of two 5-node balanced graphs. Node 1 has turned from blue to red from the left to the right.}
\label{fig:pseudo-bipartite}
\end{center}
\end{figure}

Consider an irreducible, balanced signed graph $\cG(\cN,\cE^+ \cup \cE^-,\cU)$ with nodes $\cN$, positive and negative inter-node edges, $\cE^+$ and $\cE^-$, and self-loops $\cU$. 
According to CHT, nodes $\cN$ can be partitioned into blue and red clusters, $\cN_b$ and $\cN_r$, such that \begin{enumerate}
\item $(i,j) \in \cE^+$ implies that either $i,j \in \cN_b$ or $i,j \in \cN_r$.
\item $(i,j) \in \cE^-$ implies that either $i \in \cN_b$ and $j \in \cN_r$, or $i \in \cN_r$ and $j \in \cN_b$.
\end{enumerate}
As an example, consider the 5-node balanced signed graph in Fig.\;\ref{fig:pseudo-bipartite}\,(left), where nodes 1, 2 and 5 are colored blue, while nodes 3 and 4 are colored red.
Only positive edges connect same-color node pairs, while negative edges connect opposite-color node pairs. 
There is no cycle of odd number of negative edges.



We now state a generalization of Theorem 1 to balanced signed graphs as follows:
\begin{theorem}
Denote by $\M$ a generalized graph Laplacian matrix coresponding to a balanced, irreducible signed graph $\cG$.
Denote by $\v$ the first eigenvector of $\M$ corresponding to the smallest eigenvalue $\lambda_{\min}$.
Define $\B = \S \M \S^{-1}$ as a similarity transform of $\M$, where $\S = \text{diag}(s_1, \ldots, s_K)$.
If $s_i = 1/v_i$, where $v_i \neq 0, \forall i$, then Gershgorin disc left-ends of $\B$ are aligned at $\lambda_{\min}$, \ie, $B_{ii} - \sum_{j \,|\, j\neq i} |B_{ij}| = \lambda_{\min}, \forall i$. 
\label{thm:GDA2}
\end{theorem}
We prove Theorem \ref{thm:GDA2} as follows.

\begin{proof}
We first reorder blue nodes before red nodes in the rows and columns of $\M$, so that $\M$ can be written as a $2 \times 2$ block matrix as follows:
\begin{align}
\M = \left[ \begin{array}{cc}
\M_{11} & \M_{12} \\
\M_{12}^{\top} & \M_{22}
\end{array} \right]
\end{align}
where off-diagonal terms in $\M_{11}$ ($\M_{22}$) are negative stemming from positive edge weights $w_{ij} \geq 0$ connecting same-color nodes, and entries in $\M_{12}$ are positive stemming from negative edge weights $w_{ij} \leq 0$ connecting different-color nodes.
Define now a similarity transform $\M'$ of $\M$:
\begin{align}
\M' &= \left[ \begin{array}{cc}
\I_b & \0 \\
\0 & -\I_r
\end{array} \right]
\left[ \begin{array}{cc}
\M_{11} & \M_{12} \\
\M_{12}^{\top} & \M_{22}
\end{array} \right]
\left[ \begin{array}{cc}
\I_b & \0 \\
\0 & -\I_r
\end{array} \right] \\
&= \left[ \begin{array}{cc}
\M_{11} & -\M_{12} \\
-\M_{12}^{\top} & \M_{22}
\end{array} \right].
\end{align}
We interpret $\M' = \D - \W_g'$ as a generalized graph Laplacian matrix for a new graph $\cG'(\cN,\cE',\cU')$ derived from $\cG$, where $\cG'$ retains positive edges $\cE^+$, but for each negative edge $\cE^-$, $\cG'$ switches its sign to positive.
Thus,
\begin{align}
w'_{ij} = \left\{ \begin{array}{ll}
w_{ij} & \mbox{if}~ (i,j) \in \cE^+ \\
-w_{ij} & \mbox{if}~ (i,j) \in \cE^- 
\end{array} \right. .
\end{align}
To retain the same degree matrix $\D$ as $\cG$, we assign a self-loop for each node $i$ in $\cG'$ with weight $u'_{i}$ defined as
\begin{align}
u'_{i} &= u_{i} + 2 \sum_{j \,|\, (i,j) \in \cE^-} w_{ij}.
\end{align}

As a similarity transform, $\M$ and $\M'$ have the same eigenvalues, and an eigenvector $\z$ for $\M'$ maps to an eigenvector $\v$ for $\M$ as follows:
\begin{align}
\v = \left[ \begin{array}{cc}
\I_b & \0 \\
\0 & -\I_r
\end{array} \right] \z .
\label{eq:eigVecMap}
\end{align}

Finally, we define a \textit{shifted} graph Laplacian matrix $\M'' = \M' + \epsilon \I$, where constant $\epsilon > 0$ is
\begin{align}
\epsilon &> \max_{i} \left\{ 
- \sum_{j \,|\, (i,j) \in \cE^+ \cup \cE^-} w_{ij} - u_i
\right\} .
\label{eq:append0}
\end{align}
$\M''$ has the same set of eigenvectors as $\M'$, and its eigenvalues are the same as $\M'$ but offset by $\epsilon$. 

$\M''$ has strictly positive node degrees $D''_{ii}$, \ie, 
\begin{align}
D''_{ii} &= \sum_{j | (i,j) \in \cE'} w'_{ij} + u'_{i} + \epsilon \nonumber \\
&= \sum_{j | (i,j) \in \cE^+} w_{ij} -
\sum_{j |(i,j) \in \cE^-} w_{ij} +
u_i + 2 \sum_{j | (i,j) \in \cE^-} w_{ij} + \epsilon \nonumber \\
&= \sum_{j|(i,j) \in \cE^+} w_{ij} +
\sum_{j | (i,j) \in \cE^-} w_{ij} + u_i + \epsilon \stackrel{(a)}{>} 0 \nonumber .
\end{align}
$(a)$ is due to the assumed inequality for $\epsilon$ in \eqref{eq:append0}.

From Theorem \ref{thm:GDA}, given $\M''$ is a Laplacian matrix for an irreducible positive graph, first eigenvector of $\M''$ (also first eigenvector of $\M'$) $\z$ is a strictly positive vector. 
Thus, from \eqref{eq:eigVecMap}, corresponding $\v$ for $\M$ is a strictly non-zero first eigenvector, \ie, $v_i \neq 0, \forall i$. 

Having established first eigenvector $\v$ of $\M$, we define diagonal matrix $\S = \text{diag}(1/v_1, \ldots, 1/v_N)$, and write
\begin{align}
\S \M \S^{-1} \S \v &= \lambda_{\min} \S \v \\
\B \1 &= \lambda_{\min} \1
\label{eq:append2}
\end{align}
where $\B = \S \M \S^{-1}$.
Each row $i$ in (\ref{eq:append2}) states that 
\begin{align}
B_{ii} + \sum_{j \,|\, j \neq i} B_{ij} &= \lambda_{\min} \\
M_{ii} + s_i \sum_{j \,|\, j \neq i} M_{ij} / s_j &= \lambda_{\min} .
\label{eq:gda2}
\end{align}
Suppose $i$ is a red node. 
Then $v_i = -z_i < 0$, and thus $s_i = 1/v_i < 0$.
For each red neighbor $j$ of $i$, $s_j < 0$, and $w_{ij} > 0$ means that $M_{ij} < 0$.
We can hence conclude that $s_i M_{ij} / s_j < 0$ and $s_i M_{ij} / s_j  = - |s_i M_{ij} / s_j|$.
For each blue neighbor $j$ of $i$, $s_j > 0$, and $w_{ij} < 0$ means that $M_{ij} > 0$. 
We can hence conclude also that $s_i M_{i,j} / s_j < 0$ and $s_i M_{ij} / s_j  = - |s_i M_{ij} / s_j|$.
Similar analysis can be performed if $i$ is a blue node instead.
Thus (\ref{eq:gda2}) can be rewritten as 
\begin{align}
M_{ii} - \sum_{j \,|\, j \neq i} | s_i M_{ij} / s_j| = \lambda_{\min} .
\end{align}
In other words, the left-end of $\B$'s $i$-th Gershgorin disc---center $M_{ii}$ minus radius $\sum_{j \neq i} | s_i M_{ij} / s_j|$---is aligned at $\lambda_{\min}$.
This holds true for all $i$.
\end{proof}

\section{Optimizing Metric Diagonals}
\label{sec:learning}
We now use GDPA to optimize a metric matrix $\M$. 
We first define our search space of metric matrices, and our problem to optimize $\M$'s diagonal terms.
We then describe how GDPA can be used in combination with the Frank-Wolfe method to speed up our optimization.

\subsection{Search Space of Graph Metric Matrices}

We assume that associated with each data sample $i$ is a length-$K$ feature vector $\f_i \in \mathbb{R}^K$. 
A \textit{metric matrix} $\M \in \mathbb{R}^{K \times K}$ defines the feature distance $\delta_{ij}(\M)$---the \textit{Mahalanobis distance} \cite{mahalanobis1936}---between samples $i$ and $j$ as 
\begin{equation}
\delta_{ij}(\M) = (\f_i-\f_j)^{\top} \mathbf{M} (\f_i-\f_j) .
\label{eq:featureDist}
\end{equation}
The definition of a metric \cite{fsp2014} requires $\M$ to be a real, symmetric and PD matrix, denoted by $\M \succ 0$.
This means that $\delta_{ij}(\M)$ is strictly positive unless $\f_i = \f_j$, \ie, $(\f_i-\f_j)^{\top} \M (\f_i-\f_j) > 0$ if $\f_i - \f_j \neq \0$. 

To efficiently enforce $\M \succ 0$, we invoke our developed GDPA theory for generalized graph Laplacian matrices.
We first formally define the search space $\cS$ of matrices for our optimization framework as follows: 
\begin{definition}
$\cS$ is a space of real, symmetric matrices that are generalized graph Laplacian matrices corresponding to irreducible, balanced signed graphs.
\end{definition}

We call a matrix $\M \in \cS$ that is also PD a \textit{graph metric matrix}.

\subsection{Problem Formulation}

We next pose an optimization problem for $\M$: 
find an optimal graph metric matrix $\M$---leading to feature distances $\delta_{ij}(\M)$ in \eqref{eq:featureDist}---that yields the smallest value of a convex differentiable objective $Q(\{\delta_{ij}(\M)\})$, \ie, 
\begin{align}
\min_{\mathbf{M} \in \cS}
Q\left(\{\delta_{ij}(\M)\} \right),
~~~\mbox{s.t.}~~ 
\left\{ \begin{array}{l}    
\text{tr}(\M) \leq C \\
\M \succ 0
\end{array} \right.
\label{eq:optimize_M}
\end{align}
where $C > 0$ is a chosen parameter. 
Constraint $\text{tr}(\M) \leq C$ is added to upper-bound feature distances $\delta_{ij}(\M)$. 
Assuming feature vectors $\f_i$ are normalized \cite{classificationpami19} so that $\|\f_i\|_2 \leq 1$, $\forall i$, and eigen-decomposition $\M = \V \bLambda \V^{\top}$ where $\bLambda = \text{diag}(\lambda_1, \ldots, \lambda_K)$, we can bound $\delta_{ij}(\M)$ as follows:
\begin{align}
(\f_i - \f_j)^{\top} \M (\f_i - \f_j) &= (\f_i - \f_j)^{\top} \V \bLambda \V^{\top} (\f_i - \f_j)
\nonumber \\
&\leq \sum_{k=1}^K \lambda_k \alpha_k^2 
\leq \alpha_{\max}^2 \sum_{k=1}^K \lambda_k 
\nonumber \\
&= \alpha_{\max}^2 \text{tr}(\M) \leq 4 C, 
\end{align}
where $\alpha_k = \v_k^{\top} (\f_i - \f_j)$ is the inner product of eigenvector $\v_k$ and $\f_i - \f_j$---the $k$-th \textit{Graph Fourier transform} (GFT) coefficient of $\f_i - \f_j$ \cite{8347162}. 
Because eigenvectors in $\V$ are unit-norm and $\|\f_i - \f_j\|_2 \leq 2$, $\alpha_{\max} = 2$.

For stability, we assume also that the objective is lower-bounded, \ie, $\min_{\M \succ 0} Q(\{\delta_{ij}(\M)\}) \geq \kappa > -\infty$ for some constant $\kappa$.
We examine examples of objective $Q(\{\delta_{i,j}(\M)\})$ in Section\;\ref{sec:results}. 

Our strategy to solve \eqref{eq:optimize_M} is to optimize $\mathbf{M}$'s diagonal terms \textit{plus} one row/column of off-diagonal terms at a time using the Frank-Wolfe (FW) iterative method \cite{pmlr-v28-jaggi13}, where each FW iteration is solved as a linear program (LP) until convergence. 
We discuss first the initialization of $\M$, then the optimization setup for $\M$'s diagonal terms.
For notation convenience, in the sequel we write the objective simply as $Q(\M)$, with the understanding that metric $\M$ computes first feature distances $\delta_{ij}(\M)$, which in turn determines objective $Q(\{\delta_{ij}(\M)\})$.

\subsubsection{Initialization of Metric $\M$}
\label{eq:subsubsec:initial}

We initialize a valid graph metric matrix $\M^0$ as a Laplacian matrix corresponding to a \textit{tree graph} as follows:
\begin{enumerate}
\item Initialize diagonal terms as $M_{ii}^0 := C/K, \forall i$.
\item Initialize off-diagonal terms $M_{ij}^0$, $i \neq j$, by iteratively moving one node from sets $\cN_c'$ to $\cN_c$:
\begin{enumerate}
\item Initialize $\cN_c$ as a random node $i \in \cN$ and $\cN_c' = \cN \setminus \{i\}$. 
\item At each iteration, identify node $j \in \cN_c'$ with the largest correlation in magnitude, $|E_{ij}|$, in an empirical covariance matrix $\mathbf{E}$ computed from data, to a node $i \in \cN_c$.
Move $j$ to $\cN_c$, and assign $M_{ij}^0=M_{ji}^0:=\text{sign}(E_{ij})C/K^2$.
\end{enumerate}
\end{enumerate}
Step 2 is equivalent to adding one leaf node at a time to the tree graph in $\cN_c$.
Incrementally adding leaf nodes means that $\cN_c$ remains a tree, which contains no cycles.
Thus, a tree graph is always balanced.

$\mathbf{E}$ is computed from available feature vectors $\f_i, \forall i \in \{1, \ldots, N\}$, where $N$ is the number of available samples from data with feature vectors. 
Using the largest correlation magnitudes to establish the bare minimum connectivity for a connected graph minimizes the likelihood of assigning the wrong edge signs.

For example, a $4 \times 4$ initial matrix $\M^0$ corresponding to a 4-node tree with edges $(1,2), (1,3), (3,4)$ is

\vspace{-0.05in}
\begin{scriptsize}
\begin{align}
\M^0 = C/4 \left[ \begin{array}{cccc}
1 & \mbox{sign}(E_{12})/4 & \mbox{sign}(E_{13})/4 & 0 \\
\mbox{sign}(E_{12})/4 & 1 & 0 & 0 \\
\mbox{sign}(E_{13})/4 & 0 & 1 & \mbox{sign}(E_{34})/4 \\
0 & 0 & \mbox{sign}(E_{34})/4 & 1
\end{array} \right].
\nonumber 
\end{align}
\end{scriptsize}\noindent
Initialization of the diagonal terms ensures that constraint $\text{tr}(\M^0) \leq C$ is satisfied.
Initialization of the off-diagonal terms ensures that $\M^0$ is symmetric and PD---Gershgorin disc left-ends are $C/K - \sum_{i\neq j}|M^0_{ij}| > C/K - C(K-1)/K^2 > 0$. 
Thus, we can conclude that $\M^0$ is a graph metric matrix, \ie, $\M^0 \in \cS$ and $\M^0 \succ 0$.

Given $\M^0$, we specify sets of blue ($\mathcal{N}_b$) and red ($\mathcal{N}_r$) color nodes according to edge signs in tree graph $\cG$ corresponding to $\M^0$.
Specifically, we assign the initial node $i$ blue. 
Then, we assign each of its neighbors, $j$, blue (red) if sign of edge $(i,j)$ is positive (negative), then node $j$'s neighbors and so on.
This way, all nodes has colors assigned that are consistent with edge signs in $\cG$.

\subsubsection{Optimization of Diagonals}

Optimizing $\M$'s diagonal terms $M_{ii}$ alone,  \eqref{eq:optimize_M} becomes
\begin{align}
&\min_{\{M_{ii}\}} ~~
Q(\M)
\label{eq:optimize_diagonal} \\
& \text{s.t.} \quad \,\M \succ 0; \;\;\;
\sum_{i} M_{ii} \leq C; ~~~
M_{ii} > 0, \,\forall i \nonumber
\end{align}
where $\text{tr}(\M) = \sum_i M_{ii}$. 
Because the diagonals do not affect the irreducibility and balance of matrix $\M$, the only requirement for $\M$ to be a graph metric is just $\M \succ 0$.

\subsection{Replacing PD Cone with Linear Constraints}

To efficiently enforce $\M \succ 0$, we derive sufficient linear constraints using GCT \cite{gahc}.
A direct application of GCT on $\M$, as discussed in Section\;\ref{subsubsec:GCT}, is to require \textit{all} Gershgorin disc left-ends to be positive.
This translates to a linear constraint for each row $i$:
\begin{align}
M_{ii} \geq \sum_{j \,|\, j \neq i} |M_{ij}| + \rho,
~~~~~~ \forall i \in \{1, \ldots, K\}
\label{eq:GCT_linConst}
\end{align}
where $\rho > 0$ is a small parameter. 

However, as discussed in Section \ref{subsubsec:GCT}, GCT lower bound $\lambda^-_{\min} \triangleq \min_i M_{ii} - \sum_{j \neq i} |M_{ij}|$ for $\lambda_{\min}$ can be loose.
When optimizing $\M$, enforcing \eqref{eq:GCT_linConst} directly can mean a severely restricted space compared to the original $\{\M ~|~ \M \succ 0\}$ in \eqref{eq:optimize_diagonal}, resulting in an inferior solution. 

To derive more appropriate linear constraints---thus a more comparable search space to original $\{\M ~|~ \M \succ 0\}$ when solving $\min Q(\M)$---we leverage our GDPA theory and examine instead the Gershgorin discs of a similarity-transformed matrix $\B$ from $\M$, \ie, $\B = \S \M \S^{-1}$,
where $\S = \text{diag}(s_1, \ldots, s_K)$.
This leads to the following linear constraints instead:
\begin{align}
M_{ii} \geq \sum_{j \,|\, j \neq i} \left| \frac{s_i M_{ij}}{s_j} \right| + \rho, 
~~~~ \forall i \in \{1, \ldots, K\} .
\label{eq:scaled_linConst}
\end{align}

Note that \textit{any} choice of $\{s_i\}^K_{i=1}$ such that diagonal matrix $\S$ is invertible would be sufficient for constraints \eqref{eq:scaled_linConst} to enforce PDness of a solution $\M$. 
However, \textit{the crux is to select the most appropriate scalars} $\{s_i\}^K_{i=1}$. 

Suppose that the optimal solution $\M^*$ to \eqref{eq:optimize_diagonal} is known.
Then, using the first eigenvector $\v^*$ of $\M^*$ corresponding to the smallest eigenvalue $\lambda^*_{\min} > 0$, one can compute $s_i = 1/v^*_i, \forall i$, to define linear constraints \eqref{eq:scaled_linConst}. 
By Theorem \ref{thm:GDA2}, disc left-ends of $\B = \S \M^* \S^{-1}$ are aligned exactly at $\lambda^*_{\min}$, and thus $\M^*$ is a feasible solution to \eqref{eq:scaled_linConst}. 
Linear constraints \eqref{eq:scaled_linConst} are \textit{tightest possible} for solution $\M^*$, in the sense that $\lambda^-_{\min}(\B) = \lambda_{\min}(\B) =  \lambda_{\min}(\M^*)$. 

Of course, in practice we do not know the optimal solution $\M^*$ \textit{a priori}. 
Thus, we solve the optimization iteratively, where we use the previous solution $\M^t$ at iteration $t$ to compute first eigenvector $\v^t$ and then scalars $\{s_i^t\}^K_{i=1}$, solve for a better solution $\M^{t+1}$ using linear constraints \eqref{eq:scaled_linConst}, compute new scalars again etc until convergence.
Specifically,
\begin{enumerate}
\item Given scalars $s_i^t$'s, compute solution $\M^{t+1}$ minimizing objective $Q(\M)$ subject to \eqref{eq:scaled_linConst}, \ie, 
\begin{align}
\min_{\{M_{ii}\}} &
Q \left( \M \right) \label{eq:optimize_diagonal2}  \\
\text{s.t.} & ~~ M_{ii} \geq \sum_{j \,|\, j \neq i} \left| \frac{s^t_i M_{ij}}{s^t_j} \right| + \rho, \forall i;
~~~\sum_{i} M_{ii} \leq C .
\nonumber
\end{align}
\item Given computed $\M^{t+1}$, update scalars $s_i^{t+1} = 1/v_i^{t+1}$ where $\v^{t+1}$ is the first eigenvector of $\M^{t+1}$.
\item Increment $t$ and repeat until convergence.
\end{enumerate}



\subsubsection{Algorithm Convergence}

We prove convergence to a local minimum for our iterative algorithm.
We first show that after the scalars in \eqref{eq:optimize_diagonal2} are updated to $\{s_i^{t+1}\}^K_{i=1}$, previous solution $\M^{t+1}$ remains feasible to \eqref{eq:optimize_diagonal2}.
We state this formally as a lemma:

\begin{lemma}
Solution $\M^{t+1}$ at iteration $t+1$ remains a feasible solution in optimization \eqref{eq:optimize_diagonal2} during the next iteration when constraints in \eqref{eq:optimize_diagonal2} are updated as $s_i^{t+1} = 1/v_i^{t+1}$, where $\v^{t+1}$ is the first eigenvector of $\M^{t+1}$.
\label{lemma:feasible}
\end{lemma}

\begin{proof}
Since $\M^{t+1}$ is a feasible solution to \eqref{eq:optimize_diagonal2} for scalars $\{s_i^{t}\}^K_{i=1}$ by assumption, $\lambda_{\min}(\M^{t+1}) \geq \lambda^-_{\min}(\B) \geq \rho$ for similarity transform $\B = \S \M^{t+1} \S^{-1}$ and $\S = \text{diag}(s_1^t, \ldots, s_K^t)$.
Since $\M^{t+1} \in \cS$, by Theorem\;\ref{thm:GDA2} all its Gershgorin disc left-ends can be aligned at $\lambda_{\min}(\M^{t+1})$ using scalars $\{s_i^{t+1}\}^K_{i=1}$, where $s_i^{t+1} = 1/v_i^{t+1}$ and $\v^{t+1}$ is the first eigenvector of $\M^{t+1}$. 
Thus, we can write
\begin{align}
M_{ii}^{t+1} - \sum_{j\,|\,j \neq i} \left| \frac{s^{t+1}_i M_{ij}^{t+1}}{s^{t+1}_j} \right| &= \lambda_{\min}(\M^{t+1}) 
\nonumber \\
&\geq \lambda^-_{\min}(\B) \geq \rho .
\nonumber 
\end{align}
Since $\M^{t+1}$ satisfies all constraints in \eqref{eq:optimize_diagonal2} using scalars $\{s_i^{t+1}\}^K_{i=1}$, it is a feasible solution.
\end{proof}

Lemma\;\ref{lemma:feasible} means that the objective $Q(\M^{t})$ is \textit{non-increasing} across iterations until local convergence in optimization variable $\M^t$.
Since the previous optimal solution remains feasible in the next GDPA-based LP iteration, $Q(\M^{t+1})\leq Q(\M^t)$.
Given that the objective is lower-bounded by $\kappa$ by assumption, this means that the iterative algorithm is guaranteed to converge and not oscillate.

This also means that when the algorithm terminates upon solution convergence at $Q(\M^{t+1})$, the converged solution $\M^o$ is optimal both for scalars $\{s^{t+1}_i\}$ at iteration $t+1$ and scalars $\{s^t_i\}$ at iteration $t$.
Thus, within a \textit{local neighborhood} of scalars $\{s_i\}$ where $s_i = \alpha s^{t+1}_i + (1-\alpha) s^{t}_i, \forall i$, where $0 \leq \alpha \leq 1$, $\M^o$ is an optimal solution, and thus a local optimum.


\vspace{0.1in}
\noindent
\textbf{Remark}:
We see the importance of a first eigenvector $\v^{t+1}$ of $\M^{t+1} \in \cS$ where $v^{t+1}_i \neq 0, \forall i$. 
If a solution $\M^{t+1} \not\in \cS$ and $\exists i$ such that $v_i = 0$, then we cannot leverage Theorem\;\ref{thm:GDA2} to guarantee the existence of \textit{tightest possible} scalars $\{s_i\}^K_{i=1}$, \ie, scalars where
$\S = \text{diag}(s_1, \ldots, s_K)$ and $\B = \S \M^{t+1} \S^{-1}$ so that $\lambda^-_{\min}(\B) = \lambda_{\min}(\M^{t+1})$. 
This means we cannot guarantee solution $\M^{t+1}$ remains feasible in the next iteration, nor the non-increasing property of our iterative algorithm that is required to ensure local minimum convergence. 



\subsubsection{First Eigenvector Computation} 
\label{sssec:fec}
The remaining issue is how to best compute first eigenvector $\v^{t+1}$ given solution $\M^{t+1}$ repeatedly.
For this task, we employ \textit{Locally Optimal Block Preconditioned Conjugate Gradient} (LOBPCG) \cite{Knyazev01}, a fast algorithm in linear algebra known to compute extreme eigen-pairs efficiently,
with complexity $O(ab)$, where $a$ denotes the number of non-zero entries in $\M$ and $b$ denotes the number of iterations till convergence.
Because LOBPCG is itself iterative, it benefits from \textit{warm start}: algorithm converges much faster if a good solution initiates the iterations. 
In our case, we use previously computed eigenvector $\v^{t}$ as an initial solution to speed up LOBPCG when computing $\v^{t+1}$, reducing its complexity substantially.

\subsection{Frank-Wolfe Method}

\subsubsection{FW Step 1: Solving LP}

To solve \eqref{eq:optimize_diagonal2}, we employ the Frank-Wolfe (FW) method \cite{pmlr-v28-jaggi13}. 
The first FW step linearizes the objective $Q(\M)$ using its gradient $\nabla Q(\M^t)$ with respect to diagonal terms $\{M_{ii}\}$, computed using previous solution $\M^t$, \ie,
\begin{align}
\nabla Q(\M^t) = \left. \left[ \begin{array}{c}
\frac{\partial Q(\M)}{\partial M_{1,1}} \\
\vdots \\
\frac{\partial Q(\M)}{\partial M_{K,K}}
\end{array} \right] \right|_{\M^t} .
\label{eq:gradM}
\end{align}

Given gradient $\nabla Q(\M^t)$, optimization \eqref{eq:optimize_diagonal2} becomes a LP at each iteration $t$:
\begin{align}
\min_{\{M_{ii}\}} &
\mathrm{vec}(\{M_{ii}\})^\top ~\nabla Q(\M^t) 
\label{eq:fwlp} \\
\text{s.t.} & ~~ M_{ii} \geq \sum_{j \,|\, j \neq i} \left| \frac{s_i M_{ij}^t}{s_j} \right| + \rho, ~\forall i,
~~~\sum_{i} M_{ii} \leq C 
\nonumber
\end{align}
where $\mathrm{vec}(\{M_{ii}\}) = [M_{1,1} ~ M_{2,2} ~\ldots ~ M_{K,K}]^{\top}$ is a vector composed of diagonal terms $\{M_{ii}\}$, and $M_{ij}^t$ are off-diagonal terms of previous solution $\M^t$.
LP \eqref{eq:fwlp} can be solved efficiently using known fast algorithms such as Simplex \cite{co1998} and interior point method \cite{co2009}. 

\subsubsection{FW Step 2: Step Size Optimization}

The second FW step combines the newly computed solution $\{M^o_{ii}\}$ in step 1 with the previous solution $\{M^t_{ii}\}$ using step size $\gamma$, where $0 \leq \gamma \leq 1$:
\begin{align}
M^{t+1}_{ii} = M^t_{ii} + \gamma (M^o_{ii} - M^t_{ii}), ~~~
\forall i .
\end{align}
We compute the optimal step size $\gamma$ as follows. 
Define \textit{direction} $\{d_{ii} \}$ where $d_{ii} = M^o_{ii} - M^t_{ii}$.
We solve a one-dimensional optimization problem for step size $\gamma$:
\begin{align}
\min_{\gamma \,|\, 0 \leq \gamma \leq 1}
Q(\M^t + \gamma \, \text{diag}(\{ d_{ii} \})
\label{eq:StepSize0}
\end{align}
where $\text{diag}(\{d_{ii}\})$ is a diagonal matrix with $\{d_{ii}\}$ along its diagonal entries.
Define $\M^* = \M^t + \gamma \text{diag}(\{d_{ii}\})$.  
Using the chain rule for multivariate functions, we write
\begin{align}
\frac{\partial Q(\M^*)}{\partial \gamma} &= \sum_{i=1}^K
\frac{\partial Q(\M^*)}{\partial M^*_{ii}}
\frac{\partial M^*_{ii}}{\partial \gamma} 
\\
&= \sum_{i=1}^K
\frac{\partial Q(\M^*)}{\partial M^*_{ii}}
d_{ii}.
\label{eq:stepSize}
\end{align}
Substituting $M^*_{ii} = M^t_{ii} + \gamma d_{ii}$ into \eqref{eq:stepSize}, we can write $Q'(\gamma) = \frac{\partial Q(\M^*)}{\partial \gamma}$ as a function of $\gamma$ only. 

Since $Q(\M)$ is convex, one-dimensional $Q(\gamma)$ is also convex, and $Q'(\gamma^*) = 0$ at a unique $\gamma^*$.
In general, we cannot find $\gamma^*$ in closed form given an arbitrary $Q(\M)$. 
However, one can approximate $\gamma^*$ quickly given derived $Q'(\gamma)$ using any root-finding algorithm, such as the \textit{Newton-Raphson} (NR) method \cite{nrfwsso}.
Given the range restriction of $\gamma$ in \eqref{eq:StepSize0}, our proposed procedure to find step size $\gamma^t$ at iteration $t$ is thus the following:
\begin{enumerate}
\item Derive $Q'(\gamma)$ using \eqref{eq:stepSize} and compute minimizing $\gamma^*$ using NR.
If $Q'(\gamma)$ is a constant, then $\gamma^*$ is either 0 or 1, depending on the sign of $Q'(\gamma)$.
\item Compute appropriate step size $\gamma^t$ as follows:
\begin{align}
\gamma^t = \left\{ \begin{array}{ll}
1 & \mbox{if}~ \gamma^* > 1 \\
0 & \mbox{if}~ \gamma^* < 0 \\
\gamma^* & \mbox{o.w.}
\end{array} \right. .
\end{align}
\end{enumerate}
The updated solution from an FW iteration is then $\M^{t+1} = \M^t + \gamma^t \text{diag}(\{d_{ii}\})$.
FW step 1 and 2 are executed repeatedly until convergence.

\subsubsection{Comparing Frank-Wolfe and Proximal Gradient}

After replacing the PD cone constraint with a series of linear constraints per iteration---thus defining a (more restricted) convex feasible space $\cS^t$ that is a \textit{polytope}\footnote{A polytope is an intersection of finitely many half spaces \cite{co2009}.}---one can conceivably use \textit{proximal gradient} (PG) \cite{Parikh31} instead of FW to optimize $Q(\M)$.
PG alternately performs a gradient descent step followed by a proximal operator that is a projection back to $\cS^t$ until convergence \cite{hu2020feature}. 
First, FW is entirely projection-free, while PG requires one convex set projection per iteration.
More importantly, it is difficult in general to determine an ``optimal" step size for gradient descent in PG---one that makes the maximal progress without overshooting. 
In the literature \cite{Parikh31}, PG step size can be determined based on Lipschitz constant of $\nabla Q(\M)$, which is expensive to compute if the Hessian matrix $\nabla^2 Q(\M)$ is large.
In contrast for FW, after direction $\{d_{ii}\}$ is determined in step 1, the objective $Q(\gamma)$ becomes one-dimensional, and thus optimal step size $\gamma$ can be identified speedily using first- and second-order information $Q'(\gamma)$ and $Q''(\gamma)$.  
In our experiments, we show that our proposed FW-based optimization is faster than a previous PG-based method \cite{hu2020feature}.


\section{Optimizing Metric Off-diagonals}
\label{sec:learning2}
Including off-diagonal terms of metric $\mathbf{M}$ into the optimization is more complicated, since changing these terms may affect the balance and connectivity of the underlying graph.
Similar to previous matrix optimization algorithms like graphical lasso \cite{friedman2008sparse},
we design a \textit{block coordinate descent} (BCD) algorithm, which optimizes one row/column of off-diagonal terms \textit{plus} diagonal terms at a time while maintaining graph balance.

\subsection{Problem Formulation}
 
First, we divide $\mathbf{M}$ into four sub-matrices:
\begin{equation}
\mathbf{M} = \begin{bmatrix}
M_{1,1} & \mathbf{M}_{2,1}^{\top} \\
\mathbf{M}_{2,1} & \mathbf{M}_{2,2} 
\end{bmatrix},
\label{eq:submatrix}
\end{equation}
where $M_{1,1} \in \mathbb{R}$, $\M_{2,1} \in \mathbb{R}^{(K-1) \times 1}$ and $\mathbf{M}_{2,2} \in \mathbb{R}^{(K-1) \times (K-1)}$. 
We optimize $\M_{2,1}$ (\ie, $\{M_{j,1}\}, \forall j \neq 1$) and $\{M_{ii}\}$ in one iteration, \ie,
\begin{align}
\min_{\M_{2,1}, \{M_{ii}\}} ~ Q(\M), ~~~
\mbox{s.t.}~~ 
\left\{ 
\begin{array}{l}
\M \succ 0 \\
\M \in \cS \\
\sum_{i} M_{ii} \leq C
\end{array} \right. .
\label{eq:optimize_off}
\end{align}
In the next iteration, a different node is selected, and with appropriate row/column permutation, we still optimize the first column off-diagonals $\M_{2,1}$ as in \eqref{eq:optimize_off}.
For $\M$ to remain a graph metric, i) $\M$ must be PD, ii) $\M$ must be balanced, and iii) $\M$ must be irreducible. 

\subsection{Maintaining Graph Balance}

We maintain graph balance during off-diagonal optimization as follows.
Assuming graph $\cG$ from previous solution $\M^{t}$ is balanced, nodes $\cN$ were already colored into blue nodes $\cN_b$ and red nodes $\cN_r$.  
Suppose now node $1$ is a blue node in new solution $\M^{t+1}$. 
We thus constrain edge weights to other blue/red nodes to be positive/negative.
Combining these sign constraints with previously discussed GDPA linear constraints to replace the PD cone constraint, the optimization becomes:
\begin{align}
\min_{\M_{2,1},\{M_{ii}\}} ~& Q(\M), 
~ \mbox{s.t.} \left\{
\begin{array}{l} 
M_{i,i} \geq \sum_{j\,|\,j\neq i} \left|\frac{s^t_i M_{ij}}{s^t_j}\right| + \rho, ~~\forall i \\
M_{i,1} \leq 0, ~~\mbox{if}~~ i \in \cN_b \\
M_{i,1} \geq 0, ~~\mbox{if}~~ i \in \cN_r \\
\sum_{i} M_{ii} \leq C
\end{array} \right. 
\label{eq:optimize_off2}
\end{align}
Note that the sign for each $s^t_i M_{i,j} / s^t_j$ is known, given we know the scalar values $s^t_i$ as well as the sign of $M_{ij}$.
Thus, the absolute value operator can be appropriately removed, and the set of constraints remain linear.

Suppose instead that node $1$ is a red node in new solution $\M^{t+1}$. 
Then the two edge sign constraints in \eqref{eq:optimize_off2} are replaced by
\begin{align}
&M_{i,1} \geq 0, ~~\mbox{if}~~ i \in \cN_b 
\nonumber \\
&M_{i,1} \leq 0, ~~\mbox{if}~~ i \in \cN_r .
\nonumber
\end{align}
After optimizing \eqref{eq:optimize_off2} twice, each time assuming node $1$ is blue/red, we retain the better solution that yields the smaller objective $Q(\M)$. 
As an example, in Fig.\;\ref{fig:pseudo-bipartite} node 1's edges to other nodes are optimized assuming it is blue/red in the left/right graph. 
In each case, weight signs of edges stemming from node 1 are constrained so that the graph remains balanced. 

\eqref{eq:optimize_off2} also has a convex differentiable objective with a set of linear constraints. 
We thus employ the discussed FW method to compute a solution.

If the color of each node remains unchanged (and hence the sign constraint in each term $M_{i,j}$, $i \neq j$, is fixed), one can also optimize the entire matrix $\M$ at once in a similar formulation as  \eqref{eq:optimize_off2}.  
In practice, we first optimize one row /column in $\M$ at a time using \eqref{eq:optimize_off2} until the node colors stabilize, then optimize the whole matrix $\M$ with fixed colors until convergence.

\subsection{Disconnected Sub-Graphs}

The previous optimization assumes that the underlying graph $\cG$ corresponding to Laplacian $\M$ is irreducible.
When optimizing off-diagonal terms in $\M$ also, $\cG$ may become disconnected into $P$ separate sub-graphs $\cG_1, \ldots, \cG_P$, with corresponding Laplacians $\M_1, \ldots, \M_P$, where $\M = \text{diag}(\M_1, \ldots \M_P)$, \ie, $\M$ is block-diagonal. 
In this case, to compute scalars $\{s_i\}^K_{i=1}$ in 
\eqref{eq:optimize_off2}, we simply compute the first eigenvector $\v_p$ for each sub-matrix $\M_p$ using LOBPCG. 
Previously discussed optimization can then be used to optimize each $\M_p$ separately.




Given initial $\M^0$ with nodes appropriately assigned to blue and red sets, $\cN_b$ and $\cN_r$, as discussed in Section\;\ref{eq:subsubsec:initial}, we summarize our optimization framework called \textit{signed graph metric learning} (SGML) in Algorithm\;\ref{alg:SGML}.


\begin{algorithm}[htp]
\begin{small}
\caption{Signed Graph Metric Learning (SGML).}
\label{alg:SGML}
\textbf{Input}: initial $\M^0$, blue \& red node sets $\cN_b$ \& $\cN_r$. \\
\textbf{Output}: $\M^*$.
\begin{algorithmic}[1]
\State Compute scalars $\{s^t\}$ via LOBPCG.
\State \textbf{while} \textit{not converged} \textbf{do}
\State $~~~$\textbf{for} $j=1:K$
\State $~~~~~~$ Assume $j \in \cN_b$.
\State $~~~~~~$ Solve \eqref{eq:optimize_off2} via FW for $\{M_{ij}\}$, $i \neq j$, and $\{M_{ii}\}$, $\forall i$.
\State $~~~~~~$ Assume $j \in \cN_r$.
\State $~~~~~~$ Solve \eqref{eq:optimize_off2} via FW for $\{M_{ij}\}$, $i \neq j$, and $\{M_{ii}\}$, $\forall i$.
\State $~~~~~~$ Choose the better of two previous solutions.
\State $~~~~~~$ Update $\cN_b$ and $\cN_r$.
\State $~~~~~~$ Update scalars $\{s^t\}$ via LOBPCG. 
\State $~~~$ \textbf{end for}
\State \textbf{end while}
\State \textbf{while} \textit{not converged} \textbf{do}
\State $~~~$ Solve \eqref{eq:optimize_off2} via FW for $\M$ while fixing $\cN_b$ and $\cN_r$.
\State $~~~$ Update scalars $\{s^t\}$ via LOBPCG. 
\State \textbf{end while}
\State \textbf{return} $\M^*$.
\end{algorithmic}
\end{small}
\end{algorithm}

\subsection{Local Optimality of Solution}

Our algorithm converges to a local minimum that may not be a global minimum because our search space $\cS$ is not a convex set.
Consider a convex combination $\M' = \alpha \M_1 + (1-\alpha) \M_2$, where $0 < \alpha < 1$, of two graph metric matrices $\M_1, \M_2 \in \cS$ corresponding to two balanced signed graphs $\cG_1$ and $\cG_2$, where  their edge signs are not the same.
In general, signed graph $\cG'$ associated with $\M'$ may not be balanced, and thus $\M' \not\in \cS$. 

However, matrices $\M \in \cS$ of the same edge signs do form a convex set.
Thus, $\cS$ is a union of convex sets of graph metric matrices of the same signs, and $\cS$ is locally convex. 
We will show in Section\;\ref{sec:results} that using initialization in Section\;\ref{eq:subsubsec:initial}, our algorithm achieved competitive objective function values for 17 different datasets.

\section{Experiments}
\label{sec:results}

We first show that a strong low-rank assumption on metric $\M$ such as \cite{10.1145/2184319.2184343} is not always desirable and can worsen the objective $Q(\M)$ unnecessarily.
We then compare our SGML optimization framework against other general optimization schemes in terms of: 
1) converged objective values and running time for various convex and differentiable objectives $Q(\M)$'s, and
2) performance in binary classification.

\subsection{Comparison with Low-Rank Prior}

\begin{table}[]
\begin{center}
\begin{small}
\caption{LMNN objective $Q(\M^*)$ by minimizing $Q(\M)+\tau||\M||_*$ using PD-cone, and objective $Q(\M)$ by minimizing $Q(\M)$ directly using SGML, for
\textit{Sonar} dataset ($60$ original features and $10$ PCA-transformed features).}
\label{tab:low_rank_experiments_before_after_PCA}
\begin{tabular}{ccccc}
\hline
\begin{tabular}[c]{@{}c@{}}schemes\\      features\end{tabular} & $\tau$ & obj. & rank & time (s)\\ \hline
 \multirow{5}{*}{\begin{tabular}[c]{@{}c@{}}PD-cone\\      original\end{tabular}} & 0 & 4.36E+01 & 35 & 1.23E+00 \\
 & 1.00E-03 & 6.53E+01 & 13 & 2.04E+00 \\
 & 2.00E-03 & 1.62E+02 & 8 & 2.54E+00  \\
 & 3.00E-03 & 3.39E+02 & 6 & 3.17E+00  \\
 & 4.00E-03 & 6.29E+02 & 5 & 3.81E+00  \\ \hline 
 \multirow{5}{*}{\begin{tabular}[c]{@{}c@{}}PD-cone\\      PCA\end{tabular}} & 0 & 4.19E+02 & 10 & 9.67E-02 \\
  & 1.00E-03 & 4.61E+02 & 10 & 1.09E+00  \\
  & 2.00E-03 & 1.75E+03 & 4 & 1.73E+00  \\
  & 3.00E-03 & 2.75E+03 & 0 & 2.03E+00  \\
  & 4.00E-03 & 2.77E+03 & 0 & 1.58E+00 \\ \hline
 \begin{tabular}[c]{@{}c@{}}SGML\\PCA\end{tabular}& - & 5.49E+02 & 10 & 2.04E-01
 \\ \hline
\end{tabular}
\end{small}
\end{center}
\end{table}

Instead of promoting good solutions, we first show that a low-rank prior added to an objective $Q(\M)$ in LMNN \cite{weinNNberger09LMNN} can worsen the solution quality noticeably.
Specifically, we added a weighted \textit{nuclear norm} $\|\M\|_*$ (the sum of singular values) to $Q(\M)$, which is the convexification of the rank of matrix $\M$ \cite{10.1145/2184319.2184343},  before optimization.
A low-rank matrix implies that there exist redundant features that are linear combinations of other features; a low-rank prior does not promote a diagonal-only metric matrix (which has full rank). 
Thus, in a scenario where the features $\f_i \in \mathbb{R}^K$ are not redundant, a low-rank prior would perform poorly. 

We see in Table\;\ref{tab:low_rank_experiments_before_after_PCA} that for dataset \textit{Sonar} with only $60$ features with little redundancy, increasing the weight $\tau$ of the nuclear norm $\|\M\|_*$ worsened the resulting objective $Q(\M^*)$ of the computed optimal solution $\M^*$. 
When the feature dimension was reduced via PCA (often done to reduce complexity in subsequent steps \cite{Wright-Ma-2021}), we see that the resulting objective worsened even faster as the nuclear norm weight increased. 
This shows that a low-rank prior making a strong assumption on feature redundancy is not always suitable.

In contrast, our SGML framework method makes no assumption on feature redundancy.
Further, unlike the low-rank prior that requires singular value decomposition and soft-thresholding of singular values per iteration for the \textit{proximal operator} \cite{Parikh31} of the nuclear norm, SGML requires only computation of the smallest eigen-pair $(\lambda_{\min}, \v)$ per iteration via LOBPCG, resulting in significant speedup.

\subsection{Comparison with Optimization Schemes}
\label{ssec:against_competing_opt}

We first compare computed objective values using SGML against three general optimization schemes:
1) standard gradient descent with projection onto a PD cone for full $\M$ optimization,
2) a recent metric learning scheme using block coordinate descent with proximal gradient (PG), adopting restricted search spaces that are intersections of half spaces, boxes and norm balls (HBNB) \cite{hu2020feature}, and
3) our previous work that is also based on GDPA but within a positive graph metric space (PGML) \cite{yang20}.

\begin{table*}[htb]
\begin{center}
\caption{Tested convex and (partially) differentiable objective functions $Q(\M)$'s. $d_{\M}(i,j)=\Delta\f_{ij}^{\top} \M \Delta\f_{ij}$.}
\label{tab:tested_Qms}
\begin{scriptsize}
\begin{tabular}{|c|c|c|c|c|}
\hline
MCML & DEML  & LSML & LMNN & GLR \\ \hline
\begin{tabular}[c]{@{}c@{}}
$\displaystyle{\sum_{i,j:y_{j}=y_{i}}d_{\M}(i,j)}$ \\ $\displaystyle{+\sum_{i}\log\sum_{k\neq i}\exp\left\{-d_{\M}(i,k)\right\}}$
\end{tabular} 
&
$\displaystyle{\sum_{\f_{i},\f_{j}\in\mathcal{D}}  \sqrt{d_{\M}(i,j)} }$ 
&
\begin{tabular}[c]{@{}c@{}} 
$\displaystyle{\sqrt{d_{\M}(a,b)}>\sqrt{d_{\M}(c,d)}.}$
\\
$\displaystyle{\sum_{\f_{a},\f_{b}\in\mathcal{S},\f_{c},\f_{d}\in\mathcal{D}} \Big(\sqrt{d_{\M}(a,b)}}$
\\ 
$\displaystyle{-\sqrt{d_{\M}(c,d)}\Big)^2 }$
\end{tabular}
& 
\begin{tabular}[c]{@{}c@{}c@{}} 
$\displaystyle{(1-\mu)\sum_{i,j\rightsquigarrow i}  d_{\M}(i,j)}$
\\ 
$\displaystyle{+\mu\sum_{i,j\rightsquigarrow i}\sum_{l}(1-y_{il})\Big[1}$
\\
$\displaystyle{+d_{\M}(i,j)-d_{\M}(i,l)\Big]_{+}}$
\end{tabular}
&
$\displaystyle{\sum_{i,j} \exp \left\{ -d_{\M}(i,j) \right\} (z_i - z_j)^2}$ 
\\ \hline
\end{tabular}
\end{scriptsize}
\end{center}
\end{table*}

\begin{table}[htb]
\begin{center}
\caption{Optimization parameters and convergence thresholds. GD=gradient descent.}
\label{tab:optimizaiton_set}
\begin{scriptsize}
\begin{tabular}{|c|c|c|c|c|} 
\hline
scheme & PD-cone & HBNB & PGML & SGML \\ 
\hline
\multicolumn{5}{|c|}{\textbf{ optimization parameters }} \\ 
\hline
trace constraint $C$ & \multicolumn{4}{c|}{$K$} \\ 
\hline
linear constraint $\rho$ & \multicolumn{2}{c|}{-} & \multicolumn{2}{c|}{0} \\ 
\hline
GD step size & \multicolumn{2}{c|}{\begin{tabular}[c]{@{}c@{}}$t_0=0.1/N.$\\$t_{k}=1.01t_{k-1},$\\if GD yields better obj.\\$t_{k}=t_{k-1}/2,\mbox{o.w.}$~~\end{tabular}} & \multicolumn{2}{c|}{-} \\ 
\hline
\multicolumn{5}{|c|}{\textbf{ convergence thresholds }} \\ 
\hline
main tol. & \multicolumn{4}{c|}{1.00E-05} \\ 
\hline
max main iter. & \multicolumn{4}{c|}{1.00E+03} \\ 
\hline
dia/offdia tol. & \multirow{4}{*}{-} & \multicolumn{3}{c|}{1.00E-03} \\ 
\cline{1-1}\cline{3-5}
max dia/offdia/FW iter. &  & \multicolumn{3}{c|}{1.00E+03} \\ 
\cline{1-1}\cline{3-5}
LOBPCG tol. &  & \multicolumn{3}{c|}{1.00E-04} \\ 
\cline{1-1}\cline{3-5}
max LOBPCG iter. &  & \multicolumn{3}{c|}{2.00E+02} \\ 
\hline
LP optimality tol. & \multicolumn{2}{c|}{\multirow{3}{*}{-}} & \multicolumn{2}{c|}{1.00E-02} \\ 
\cline{1-1}\cline{4-5}
LP interior-point tol. & \multicolumn{2}{c|}{} & \multicolumn{2}{c|}{1.00E-04} \\ 
\cline{1-1}\cline{4-5}
FW step size NR tol. & \multicolumn{2}{c|}{} & \multicolumn{2}{c|} {5.00E-01} \\
\hline
\end{tabular}
\end{scriptsize}
\end{center}
\end{table}

We evaluated PD-cone, HBNB, PGML, and SGML on the following convex and (partially) differentiable $Q(\M)$'s for $\M \succ 0$ (also in Table\;\ref{tab:tested_Qms}):
\begin{enumerate}
\item Maximally collapsing metric learning (MCML) \cite{globerson2006metric}.
\item Seminal distance metric learning (DEML) \cite{ericdml}.
For the sake of comparison without losing validity of the optimization results, we relax the constraint $\sum_{\f_{i},\f_{j}\in\mathcal{S}}  \Delta\f_{ij}^{\top} \M \Delta\f_{ij}\leq c, c>0$ ($\mathcal{S}$ denotes the set of sample pairs that have the same labels) since solving this constrained problem (solving a sparse system of linear equations) may result in $\M$ not being PD. 
See \cite{ericdml} for details.
\item Least squared-residual metric learning (LSML) \cite{LSML}. 
We set the distance weights to be all 1's and no prior metric matrix is given.
\item Large margin nearest neighbor (LMNN) \cite{weinNNberger09LMNN}. 
Note that the objective function is piecewise linear.
\item Graph Laplacian regularizer (GLR) \cite{8347162,pang2017graph}:
\begin{equation}
\sum_{i=1}^{N} \sum_{j=1}^{N} \exp \left\{ -(\f_i-\f_j)^{\top} \M (\f_i - \f_j) \right\} (z_i - z_j)^2.
\label{eq:GLR}
\end{equation}
A small GLR means that signal $\z$ at connected node pairs $(z_i, z_j)$ is similar for a large edge weight $ w_{ij} $, \ie, $\z$ is \textit{smooth} with respect to a graph $\cG$ with edge weights $w_{ij} = \exp \left\{ -(\f_i-\f_j)^{\top} \M (\f_i - \f_j) \right\}$. 
\end{enumerate}

We evaluated PD-cone, HBNB, PGML, and SGML using 17 datasets, including 14 out of 17 in \cite{classificationpami19}, \textit{Sonar} with 60 features, \textit{Madelon} with 500 features, and \textit{Colon-cancer} with 2000 features, all of which are binary datasets available in UCI\footnote{\url{https://archive.ics.uci.edu/ml/datasets.php}} and LibSVM\footnote{\url{https://www.csie.ntu.edu.tw/~cjlin/libsvmtools/datasets/binary.html}}.
For each optimization scheme, we randomly split (with random seed 0) a dataset into $T=\mbox{round}(N/4)$ folds, ran optimization on each fold and took the average of the converged objective values.
We ran similar experiments on datasets \textit{Madelon} and \textit{Colon-cancer}, except that we only ran the first 10 out of $T$ folds of the data and took the average. 
We applied the same data normalization scheme in \cite{classificationpami19} that 
1) subtracts the mean and divides by the standard deviation feature-wise, and
2) normalizes to unit length sample-wise.
We added $10^{-12}$ noise 
to the dataset to avoid NaN's due to data normalization on small samples.


The optimization parameters and convergence thresholds \cite{Bertsekas/99} of PD-cone, HBNB, PGML and SGML are listed in Table \ref{tab:optimizaiton_set}.
Finding a step size for PG based on Lipschitz constant for \textit{Madelon} and \textit{Colon-cancer} is computationally infeasible in a consumer-level machine, where Hessian $\nabla^2 Q(\M)$'s have $500^4$ and $2000^4$ entries, respectively.
Thus, as done in \cite{weinNNberger09LMNN}, the step size of gradient descent (GD) for PD-cone and HBNB was heuristically initialized as $0.1/N$, increased by 1\% if GD yielded a better objective value, and decreased by half otherwise.
For PGML and SGML, we solved LP's using Matlab linprog in Steps 5 and 7 of Algorithm \ref{alg:SGML} and Gurobi Matlab interface\footnote{\url{https://www.gurobi.com/documentation/9.0/examples/linprog_m.html}} in Step 14, both of which employ the interior-point method for solution \cite{co2009}.

As shown in Tables \ref{tab:objtime} and \ref{tab:objtime_c}, 
SGML achieved the smallest (for MCML, LSML, LMNN and GLR minimization problems) and the largest (for DEML maximization problem) averaged objective values compared to HBNB and PGML, \textit{i.e.}, the closest objective values compared to high-complexity baseline PD-cone.
As shown in the last column of Tables \ref{tab:objtime} and \ref{tab:objtime_c}, SGML performed overall better than HBNB and PGML, both of which have more restrictive search space, resulting in sub-optimal solutions for MCML, DEML, LSML and LMNN objectives.
LSML objective contains boolean expressions and LMNN objective is piecewise linear, and thus they are not differentiable everywhere; SGML still achieved highly competitive objective values compared to HBNB and PGML for LSML. 
The difference between the largest and smallest eigenvalues of the underlying $\M$ for LMNN and GLR might be smaller than MCML, DEML and LSML, which makes the norm-ball projection in HBNB particularly suitable for LMNN and GLR, resulting in competitive objective values in 11 and 10 out of 17 datasets against PGML and SGML, respectively.


All four optimization schemes were implemented in Matlab\footnote{code available: \url{https://github.com/bobchengyang/SGML}}.
Fig. \ref{fig:madelon_time} and \ref{fig:coloncancer_time} show the total running time, the running time for the highest time-complexity components of PD-cone, HBNB and SGML, and the speed gain of SGML over PD-cone, on datasets \textit{Madelon} and \textit{Colon-cancer}. 
Both figures show that 1) eigen-decomposition for PD-cone on large matrices entailed high computation complexity,
2) it often took large numbers of iterations for PD-cone and HBNB to converge using a heuristic gradient descent step size selection (see Table \ref{tab:optimizaiton_set}), 
3) SGML/HBNB benefited from LOBPCG for fast first eigenpair computation, and 
4) SGML in addition benefited from the empirically observed sparsity 
of computed $\M$ ($a\ll K^2$ for LOBPCG complexity $\mathcal{O}(ab)$) and FW step size optimization, and thus converged much faster than PD-cone and HBNB.
In particular, SGML was 8.28x and 5.56x faster than PD-cone on \textit{Madelon} with MCML and GLR objectives, respectively, as shown in Fig. \ref{fig:madelon_time},
and 1.97x, 2.42x and 3.10x faster than PD-cone on \textit{Colon-cancer} with MCML, LSML and GLR objectives, respectively, as shown in Fig. \ref{fig:coloncancer_time}.
Our proposed SGML on \textit{Madelon} was slower than PD-cone (see Fig. \ref{fig:madelon_time}), which is due to the potential large number of FW iterations during the BCD process.
However, the highest time-complexity component LOBPCG in SGML still occupied only a very small portion of the total running time.

\subsection{Binary Classification}

Further, we evaluated SGML against competing methods on binary classification using the same 14 binary datasets in \cite{classificationpami19} and \textit{Sonar}.
Specifically, we optimized four out of five objective functions in Table\;\ref{tab:tested_Qms}, MCML, DEML, LMNN and GLR, using different optimization schemes, then built a 10-nearest neighbor classifier implemented by authors of Information-Theoretic Metric Learning\footnote{\url{http://www.cs.utexas.edu/users/pjain/itml/download/itml-1.2.tar.gz}}.
We applied the same data normalization scheme in \cite{classificationpami19} as in Section\;\ref{ssec:against_competing_opt}.
We created 10 instances of 90\% training---10\% test split with random seeds 0-9, \textit{i.e.}, 10-fold random cross validation \cite{10.5555/1671238}, and computed the average accuracy.
We compared SGML against PD-cone, HBNB and PGML, in terms of the average classification accuracy of all 15 tested datasets. 
Tables \ref{tab:binaryclassification} and \ref{tab:binaryclassification_c} show that, on average, our SGML achieved better classification accuracy than our previous PGML for all four objective functions.
Furthermore, our SGML achieved comparable classification accuracy against PD-cone and HBNB, while SGML has a much lower computation complexity compared to both PD-cone and HBNB.

\begin{table*}[]
\begin{center}
\caption{Converged objective values. 
All problems minimize $Q(\M)$'s except DEML. 
Best objective values in bold (excluding PD-cone). 
Avg. of $T=\mbox{round}(N/4)$ runs.
Data is split into $T$ folds and run $T$ times.
Timed experiments on \textit{madelon} and \textit{colon-cancer} are run on the first 10 out of $T$ folds and then take the avg.
Machine spec: AMD Ryzen Threadripper 3960X 24-core processor 3.80 GHz Windows 10 64bit 128GB of RAM.}
\label{tab:objtime}
\begin{scriptsize}
\begin{tabular}{ccccccccccc}
\hline
\multirow{2}{*}{$Q(\M)$} & \multirow{2}{*}{\begin{tabular}[c]{@{}c@{}}dataset\\ $(N,K)$\end{tabular}} & \multirow{2}{*}{\begin{tabular}[c]{@{}c@{}}Australian\\ (690,14)\end{tabular}} & \multirow{2}{*}{\begin{tabular}[c]{@{}c@{}}Breastcancer\\ (683,10)\end{tabular}} & \multirow{2}{*}{\begin{tabular}[c]{@{}c@{}}Diabetes\\ (768,8)\end{tabular}} & \multirow{2}{*}{\begin{tabular}[c]{@{}c@{}}Fourclass\\ (862,2)\end{tabular}} & \multirow{2}{*}{\begin{tabular}[c]{@{}c@{}}German\\ (1000,24)\end{tabular}} & \multirow{2}{*}{\begin{tabular}[c]{@{}c@{}}Haberman\\ (206,3)\end{tabular}} & \multirow{2}{*}{\begin{tabular}[c]{@{}c@{}}Heart\\ (270,13)\end{tabular}} & \multirow{2}{*}{\begin{tabular}[c]{@{}c@{}}ILPD\\ (583,10)\end{tabular}} & \multirow{2}{*}{\begin{tabular}[c]{@{}c@{}}Liverdisorders\\ (345,6)\end{tabular}} \\
 &  &  &  &  &  &  &  &  &  &  \\ \hline
\multirow{4}{*}{MCML} & PD-cone & 5.21E-03 & 1.17E-02 & 3.46E-02 & 9.11E-01 & 1.12E-02 & 2.86E-01 & 2.62E-03 & 3.78E-02 & 1.85E-01 \\
 & HBNB & 3.54E-01 & 9.37E-02 & 5.79E-01 & 1.07E+00 & 2.13E-01 & 4.77E-01 & 3.57E-01 & 2.86E-01 & 8.96E-01 \\
 & PGML & 3.06E-01 & 1.73E-01 & 6.75E-01 & 1.24E+00 & 2.04E-01 & 6.91E-01 & \textbf{1.76E-01} & 4.72E-01 & 9.26E-01 \\
 & SGML & \textbf{2.03E-01} & \textbf{4.98E-02} & \textbf{5.03E-01} & \textbf{1.02E+00} & \textbf{1.67E-01} & \textbf{4.45E-01} & 1.90E-01 & \textbf{2.68E-01} & \textbf{7.02E-01} \\
  \hline
\multirow{4}{*}{DEML} & PD-cone & 1.62E+01 & 1.75E+01 & 1.16E+01 & 6.63E+00 & 1.80E+01 & 6.47E+00 & 1.57E+01 & 1.14E+01 & 9.64E+00 \\
 & HBNB & 8.81E+00 & 9.06E+00 & 8.09E+00 & 6.87E+00 & 7.56E+00 & 6.26E+00 & 8.80E+00 & 7.27E+00 & 8.04E+00 \\
 & PGML & 9.62E+00 & 8.03E+00 & 8.29E+00 & 6.75E+00 & \textbf{9.53E+00} & 6.42E+00 & 9.81E+00 & 8.02E+00 & 8.01E+00 \\
 & SGML & \textbf{9.91E+00} & \textbf{9.67E+00} & \textbf{8.86E+00} & \textbf{6.95E+00} & 9.11E+00 & \textbf{6.59E+00} & \textbf{9.89E+00} & \textbf{8.23E+00} & \textbf{8.36E+00} \\
  \hline
\multirow{4}{*}{LSML} & PD-cone & 8.56E-03 & 1.25E-03 & 5.68E-03 & 4.73E-01 & 1.71E-02 & 7.54E-02 & 7.57E-03 & 1.43E-02 & 8.80E-03 \\
 & HBNB & 2.32E-02 & \textbf{1.03E-03} & \textbf{2.40E-02} & 5.20E-01 & 3.27E-02 & 1.24E-01 & 1.49E-02 & 6.09E-02 & 6.08E-02 \\
 & PGML & 1.57E-01 & 1.30E-02 & 3.30E-01 & 1.96E+00 & 1.21E-01 & 9.40E-01 & 1.36E-01 & 4.59E-01 & 4.45E-01 \\
 & SGML & \textbf{6.01E-03} & 1.27E-03 & 2.51E-02 & \textbf{4.58E-02} & \textbf{3.42E-03} & \textbf{9.99E-02} & \textbf{4.02E-03} & \textbf{1.90E-02} & \textbf{3.90E-02} \\
  \hline
\multirow{4}{*}{LMNN} & PD-cone & 7.33E+00 & 6.72E+00 & 6.55E+00 & 8.64E+00 & 6.21E+00 & 6.49E+00 & 7.38E+00 & 6.12E+00 & 6.89E+00 \\
 & HBNB & \textbf{9.17E+00} & \textbf{8.10E+00} & 9.69E+00 & \textbf{9.45E+00} & 7.93E+00 & \textbf{8.62E+00} & 9.50E+00 & \textbf{8.54E+00} & \textbf{9.69E+00} \\
  & PGML & 1.08E+01 & 8.75E+00 & 9.41E+00 & 9.81E+00 & 9.70E+00 & 8.40E+00 & 1.07E+01 & 9.17E+00 & 9.86E+00 \\
 & SGML & 9.21E+00 & 8.37E+00 & \textbf{9.29E+00} & 9.75E+00 & \textbf{7.75E+00} & 8.63E+00 & \textbf{9.18E+00} & 8.55E+00 & 1.01E+01 \\
 \hline
\multirow{4}{*}{GLR} & PD-cone & 3.50E-03 & 1.13E-03 & 3.79E-02 & 1.66E+00 & 1.04E-03 & 6.29E-01 & 1.34E-03 & 2.58E-02 & 2.15E-01 \\
 & HBNB & \textbf{1.63E-01} & \textbf{5.88E-02} & \textbf{2.28E-01} & \textbf{1.57E+00} & 1.28E-01 & 6.78E-01 & \textbf{1.36E-01} & \textbf{2.02E-01} & \textbf{4.13E-01} \\
 & PGML & 4.31E-01 & 3.64E-01 & 5.55E-01 & 1.67E+00 & 1.51E-01 & 7.58E-01 & 4.17E-01 & 3.62E-01 & 6.63E-01 \\ 
 & SGML & 1.86E-01 & 8.17E-02 & 3.47E-01 & 1.60E+00 & \textbf{8.76E-02} & \textbf{6.52E-01} & 2.00E-01 & 2.33E-01 & 4.86E-01 \\
 \hline
\end{tabular}
\end{scriptsize}
\end{center}
\end{table*}


\begin{table*}[]
\begin{center}
\caption{Continuation of Table \ref{tab:objtime} on other datasets. 
Best objective values in bold (excluding PD-cone). 
All problems minimize $Q(\M)$'s except DEML.
The last second column shows the amplitude of HBNB, PGML and SGML over PD-cone.
The last column shows the number of best values for each optimization scheme (excluding PD-cone).}
\label{tab:objtime_c}
\begin{scriptsize}
\begin{tabular}{cccccccccccc} 
\hline
\multirow{2}{*}{$Q(\M)$} & \multirow{2}{*}{\begin{tabular}[c]{@{}c@{}}dataset\\ $(N,K)$\end{tabular}} & \multirow{2}{*}{\begin{tabular}[c]{@{}c@{}}Monk1\\ (556,6)\end{tabular}} & \multirow{2}{*}{\begin{tabular}[c]{@{}c@{}}Pima\\ (768,8)\end{tabular}} & \multirow{2}{*}{\begin{tabular}[c]{@{}c@{}}Planning\\ (182,12)\end{tabular}} & \multirow{2}{*}{\begin{tabular}[c]{@{}c@{}}Voting\\ (435,16)\end{tabular}} & \multirow{2}{*}{\begin{tabular}[c]{@{}c@{}}WDBC\\ (569,30)\end{tabular}} & \multirow{2}{*}{\begin{tabular}[c]{@{}c@{}}Sonar\\ (208,60)\end{tabular}} & \multirow{2}{*}{\begin{tabular}[c]{@{}c@{}}madelon\\ (2600,500)\end{tabular}} & \multirow{2}{*}{\begin{tabular}[c]{@{}c@{}}colon-cancer\\ (62,2000)\end{tabular}} & \multirow{2}{*}{\begin{tabular}[c]{@{}c@{}}over\\ PD-cone\end{tabular}} & \multirow{2}{*}{\begin{tabular}[c]{@{}c@{}}\# of\\ best\end{tabular}} \\
 &  &  &  &  &  &  &  &  &  &  &  \\ \hline
\multirow{4}{*}{MCML} & PD-cone & 1.12E-01 & 3.95E-02 & 3.52E-02 & 1.01E-02 & 9.34E-03 & 2.14E-03 & 9.95E-04 & 6.01E-03 & - & - \\
 & HBNB & 1.15E+00 & 5.41E-01 & 2.84E-01 & 1.07E-01 & 1.43E-01 & 4.36E-01 & 4.42E-01 & 4.50E-01 & 6.14E+01 & 0 \\
 & PGML & 1.04E+00 & 7.05E-01 & 4.40E-01 & 7.05E-02 & 2.75E-01 & \textbf{1.40E-01} & \textbf{4.15E-03} & 3.14E-02 & 2.06E+01 & 3 \\
 & SGML & \textbf{9.87E-01} & \textbf{5.02E-01} & \textbf{2.06E-01} & \textbf{5.14E-02} & \textbf{8.19E-02} & 2.27E-01 & 2.47E-02 & \textbf{2.01E-02} & \textbf{1.97E+01} & \textbf{14} \\
  \hline
\multirow{4}{*}{DEML} & PD-cone & 1.04E+01 & 1.15E+01 & 1.24E+01 & 1.94E+01 & 2.68E+01 & 3.17E+01 & 8.92E+01 & 1.56E+02 & - & - \\
 & HBNB & 8.30E+00 & 8.00E+00 & 7.17E+00 & 9.04E+00 & 8.79E+00 & 8.46E+00 & 8.30E+00 & 6.83E+00 & 5.58E-01 & 0 \\
 & PGML & 8.67E+00 & 8.31E+00 & \textbf{8.39E+00} & \textbf{1.06E+01} & 9.13E+00 & \textbf{1.06E+01} & \textbf{1.13E+01} & 9.29E+00 & 5.95E-01 & 5 \\
 & SGML & \textbf{8.76E+00} & \textbf{8.69E+00} & 8.11E+00 & 1.02E+01 & \textbf{1.01E+01} & 1.04E+01 & 1.11E+01 & \textbf{9.31E+00} & \textbf{6.11E-01} & \textbf{12} \\
  \hline
\multirow{4}{*}{LSML} & PD-cone & 3.23E-03 & 4.33E-03 & 5.07E-03 & 3.10E-03 & 2.09E-03 & 6.86E-03 & 2.34E-03 & 1.20E-02 & - & - \\
 & HBNB & \textbf{2.83E-02} & \textbf{2.63E-02} & 1.87E-02 & 9.27E-03 & 9.60E-03 & 1.71E-02 & 6.60E-03 & 3.50E-02 & 3.52E+00 & 4 \\
 & PGML & 5.10E-01 & 4.50E-01 & 4.24E-01 & 1.49E-01 & 5.73E-02 & 1.44E-01 & 9.07E-03 & 3.54E-02 & 3.88E+01 & 0 \\ 
 & SGML & 4.03E-02 & 2.89E-02 & \textbf{1.06E-02} & \textbf{3.53E-03} & \textbf{1.78E-03} & \textbf{2.89E-03} & \textbf{7.58E-05} & \textbf{5.25E-03} & \textbf{2.25E+00} & \textbf{13} \\
 \hline
\multirow{4}{*}{LMNN} & PD-cone & 8.09E+00 & 6.59E+00 & 6.06E+00 & 6.97E+00 & 6.92E+00 & 8.20E+00 & 7.98E+00 & 6.50E+00 & - & - \\
 & HBNB & \textbf{1.06E+01} & \textbf{9.57E+00} & \textbf{8.14E+00} & \textbf{8.35E+00} & \textbf{8.36E+00} & 9.84E+00 & 1.04E+01 & 8.70E+00 & 1.30E+00 & \textbf{11} \\
  & PGML & 1.12E+01 & 9.92E+00 & 8.36E+00 & 1.08E+01 & 9.86E+00 & 1.27E+01 & 9.55E+00 & 1.09E+01 & 1.43E+00 & 0 \\ 
 & SGML & 1.07E+01 & 9.77E+00 & \textbf{8.14E+00} & 8.77E+00 & 8.40E+00 & \textbf{9.64E+00} & \textbf{8.68E+00} & \textbf{7.22E+00} & \textbf{1.28E+00} & 7 \\
\hline
\multirow{4}{*}{GLR} & PD-cone & 1.19E-01 & 4.08E-02 & 1.90E-03 & 1.15E-03 & 1.00E-03 & 1.00E-03 & 9.88E-04 & 9.75E-04 & - & - \\
 & HBNB & \textbf{3.70E-01} & \textbf{2.58E-01} & \textbf{1.61E-01} & 8.66E-02 & 9.45E-02 & 2.62E-01 & 2.28E-01 & 1.08E-01 & 7.11E+01 & \textbf{10} \\
  & PGML & 7.48E-01 & 5.49E-01 & 2.49E-01 & \textbf{5.74E-02} & 4.28E-01 & 1.89E-01 & \textbf{8.62E-03} & \textbf{8.15E-03} & 1.04E+02 & 3 \\
 & SGML & 5.62E-01 & 3.54E-01 & 1.83E-01 & 9.85E-02 & \textbf{6.98E-02} & \textbf{1.40E-01} & 5.83E-02 & 4.54E-02 & \textbf{5.25E+01} & 4 \\
 \hline
\end{tabular}
\end{scriptsize}
\end{center}
\end{table*}

\begin{figure}
\begin{center}
\includegraphics[width=3.5in,trim=1.1in 1.1in 1in 1.2in, clip]{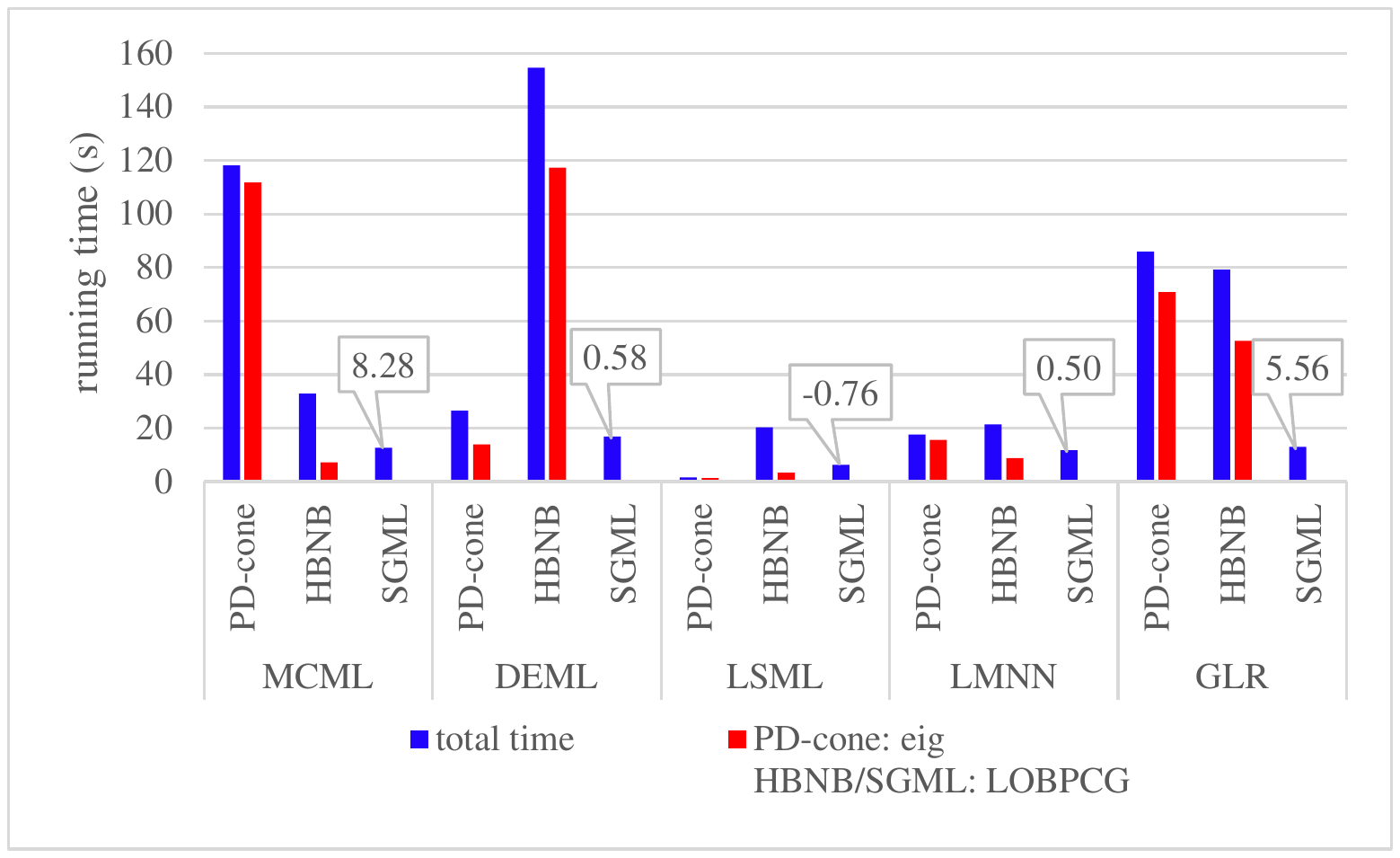}
\vspace{-0.3in}
\caption{Running time on \textit{Madelon} of optimization methods PD-cone, HBNB, and SGML on objective functions MCML, DEML, LSML, LMNN and GLR.
Labelled numbers denote the speed gain (faster (positive) or slower (negative)) of SGML over PD-cone.}
\label{fig:madelon_time}
\end{center}
\end{figure}

\begin{figure}
\begin{center}
\includegraphics[width=3.5in,trim=1in 1.1in 1in 1.2in, clip]{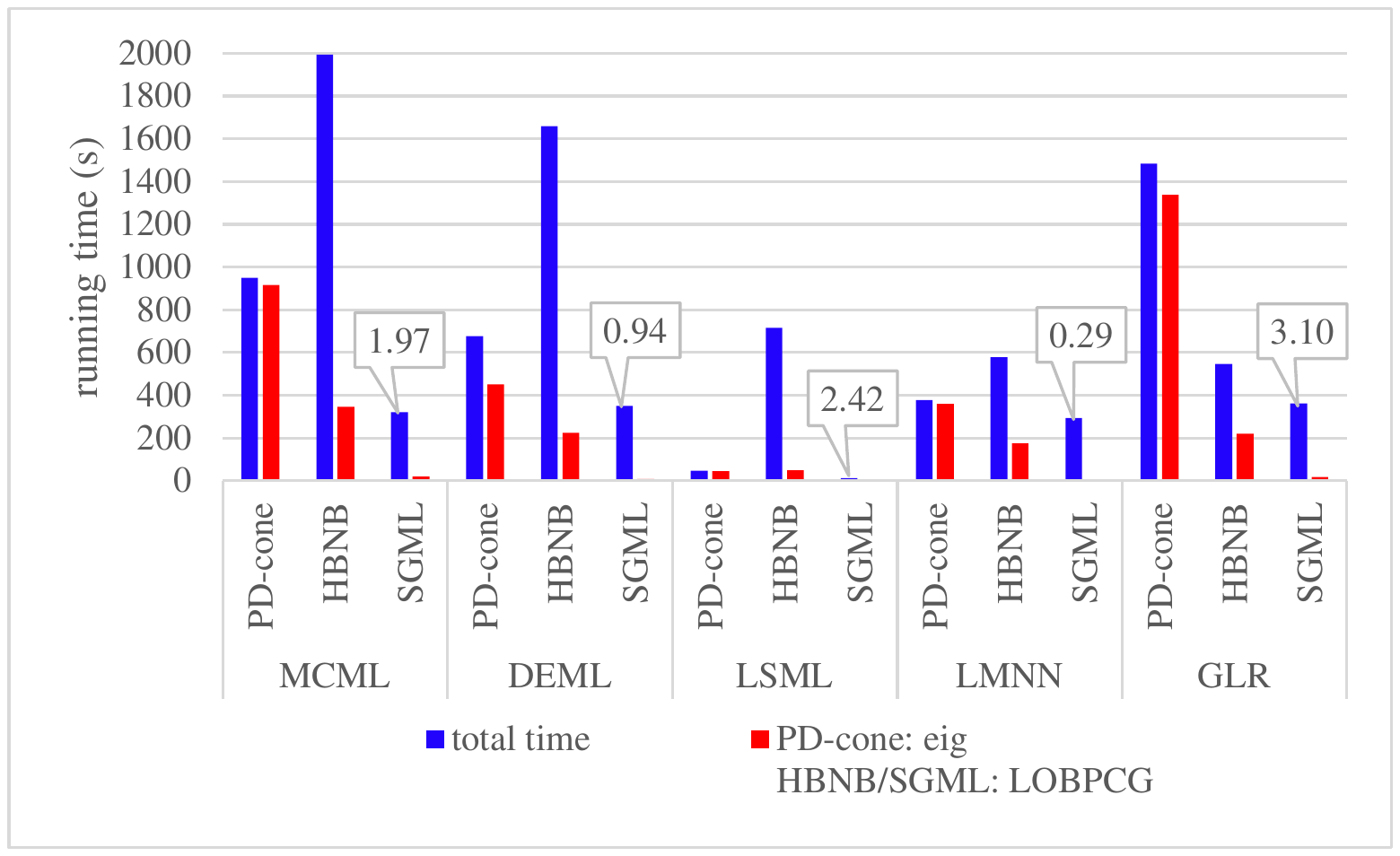}
\vspace{-0.3in}
\caption{Running time on \textit{Colon-cancer} of optimization methods PD-cone, HBNB, and SGML on objective functions MCML, DEML, LSML, LMNN and GLR.
Labelled numbers denote the speed gain (faster (positive) or slower (negative)) of SGML over PD-cone.}
\label{fig:coloncancer_time}
\end{center}
\end{figure}

\begin{table*}[]
\begin{center}
\caption{Classification accuracy (\%) with a 10-nearest neighbor classifier. Best classification accuracy in bold (excluding PD-cone). Experiments are performed by creating 10 instances of 90\% training---10\% test split with random seeds 0-9, \textit{i.e.}, 10-fold random cross validation \cite{10.5555/1671238}.}
\label{tab:binaryclassification}
\begin{scriptsize}
\begin{tabular}{cccccccccc}
\hline
\multirow{2}{*}{$Q(\M)$} & \multirow{2}{*}{\begin{tabular}[c]{@{}c@{}}dataset\\ $(N,K)$\end{tabular}} & \multirow{2}{*}{\begin{tabular}[c]{@{}c@{}}Australian\\ (690,14)\end{tabular}} & \multirow{2}{*}{\begin{tabular}[c]{@{}c@{}}Breastcancer\\ (683,10)\end{tabular}} & \multirow{2}{*}{\begin{tabular}[c]{@{}c@{}}Diabetes\\ (768,8)\end{tabular}} & \multirow{2}{*}{\begin{tabular}[c]{@{}c@{}}Fourclass\\ (862,2)\end{tabular}} & \multirow{2}{*}{\begin{tabular}[c]{@{}c@{}}German\\ (1000,24)\end{tabular}} & \multirow{2}{*}{\begin{tabular}[c]{@{}c@{}}Haberman\\ (206,3)\end{tabular}} & \multirow{2}{*}{\begin{tabular}[c]{@{}c@{}}Heart\\ (270,13)\end{tabular}} & \multirow{2}{*}{\begin{tabular}[c]{@{}c@{}}ILPD\\ (583,10)\end{tabular}} \\
 &  &  &  &  &  &  &  &  &  \\ \hline
\multirow{4}{*}{MCML} & PD-cone & 87.37 & 97.23 & 75.80 & 78.43 & 73.00 & 71.97 & 86.67 & 71.74 \\
 & HBNB & 86.50 & 97.08 & \textbf{76.05} & \textbf{78.43} & \textbf{73.70} & 72.65 & \textbf{86.67} & 62.96 \\
  & PGML & 86.80 & 97.52 & 75.92 & 78.20 & 70.90 & \textbf{75.92} & 84.07 & \textbf{64.50} \\
 & SGML & \textbf{87.09} & \textbf{97.66} & 75.40 & 78.20 & 71.60 & \textbf{75.92} & 84.44 & \textbf{64.50} \\
 \hline
\multirow{4}{*}{DEML} & PD-cone & 86.64 & 96.93 & 76.72 & 78.20 & 70.50 & 76.25 & 85.19 & 62.94 \\
 & HBNB & 85.63 & \textbf{97.23} & \textbf{76.32} & \textbf{78.43} & \textbf{73.30} & 73.30 & \textbf{85.93} & \textbf{66.06} \\
  & PGML & \textbf{86.79} & 95.77 & 75.80 & \textbf{78.43} & 69.70 & 73.65 & 82.22 & 64.12 \\
 & SGML & 84.76 & 96.35 & 72.91 & \textbf{78.43} & 72.10 & \textbf{74.63} & 84.07 & 64.84 \\
 \hline
\multirow{4}{*}{LMNN} & PD-cone & 86.35 & 97.08 & 76.19 & 78.32 & 71.90 & 75.60 & 86.30 & 64.33 \\
 & HBNB & 84.91 & 97.22 & \textbf{76.45} & 78.43 & 70.60 & \textbf{76.25} & 84.81 & 64.48 \\
 & PGML & 85.78 & \textbf{97.37} & 76.19 & \textbf{78.55} & \textbf{71.40} & \textbf{76.25} & 85.56 & 64.15 \\
 & SGML & \textbf{86.36} & \textbf{97.37} & 75.53 & 78.43 & 70.80 & 76.24 & \textbf{85.93} & \textbf{64.49} \\
  \hline
\multirow{4}{*}{GLR} & PD-cone & 87.22 & 97.51 & 76.45 & 78.20 & 71.50 & 75.92 & 85.56 & 62.09 \\
 & HBNB & 85.34 & \textbf{97.66} & 76.06 & 78.32 & 71.40 & 76.25 & \textbf{84.81} & 65.69 \\
 & PGML & \textbf{86.50} & 96.50 & 75.53 & 78.20 & 71.40 & 76.25 & 83.33 & 64.68 \\
 & SGML & 86.36 & 97.23 & \textbf{77.49} & \textbf{78.43} & \textbf{71.80} & \textbf{77.55} & 82.96 & \textbf{66.08} \\
  \hline
\end{tabular}
\end{scriptsize}
\end{center}
\end{table*}

\begin{table*}[]
\begin{center}
\caption{Continuation of Table \ref{tab:binaryclassification} on other datasets. 
Best classification accuracy in bold (excluding PD-cone).
Experiments are performed by creating 10 instances of 90\% training---10\% test split with random seeds 0-9, \textit{i.e.}, 10-fold random cross validation \cite{10.5555/1671238}.
The last column shows the avg. classification accuracy of 15 experimented datasets.}
\label{tab:binaryclassification_c}
\begin{scriptsize}
\begin{tabular}{cccccccccc}
\hline
\multirow{2}{*}{$Q(\M)$} & \multirow{2}{*}{\begin{tabular}[c]{@{}c@{}}dataset\\ $(N,K)$\end{tabular}} & \multirow{2}{*}{\begin{tabular}[c]{@{}c@{}}Liverdisorders\\ (345,6)\end{tabular}} & \multirow{2}{*}{\begin{tabular}[c]{@{}c@{}}Monk1\\ (556,6)\end{tabular}} & \multirow{2}{*}{\begin{tabular}[c]{@{}c@{}}Pima\\ (768,8)\end{tabular}} & \multirow{2}{*}{\begin{tabular}[c]{@{}c@{}}Planning\\ (182,12)\end{tabular}} & \multirow{2}{*}{\begin{tabular}[c]{@{}c@{}}Voting\\ (435,16)\end{tabular}} & \multirow{2}{*}{\begin{tabular}[c]{@{}c@{}}WDBC\\ (569,30)\end{tabular}} & \multirow{2}{*}{\begin{tabular}[c]{@{}c@{}}Sonar\\ (208,60)\end{tabular}} & \multirow{2}{*}{avg.} \\
 &  &  &  &  &  &  &  &  &  \\ \hline
\multirow{4}{*}{MCML} & PD-cone & 69.05 & 86.52 & 75.78 & 64.47 & 97.46 & 96.65 & 74.70 & 80.46 \\
 & HBNB & \textbf{67.67} & \textbf{81.67} & 74.74 & \textbf{63.42} & 96.08 & \textbf{95.95} & 75.15 & 79.25 \\
  & PGML & 63.10 & 78.61 & \textbf{76.05} & 62.95 & \textbf{96.78} & 95.77 & 74.13 & 78.75 \\
 & SGML & 63.10 & 79.34 & 75.79 & 62.95 & 95.86 & \textbf{95.95} & \textbf{76.08} & 78.93 \\
 \hline
\multirow{4}{*}{DEML} & PD-cone & 65.76 & 77.89 & 76.69 & 62.98 & 95.37 & 96.12 & 76.53 & 78.98 \\
 & HBNB & 63.62 & \textbf{78.08} & \textbf{77.08} & 61.93 & 94.93 & 95.59 & \textbf{77.03} & 78.96 \\
 & PGML & \textbf{65.71} & 77.71 & 74.48 & 63.45 & \textbf{95.16} & 94.18 & 72.82 & 78.00 \\
 & SGML & \textbf{65.71} & 76.08 & 72.14 & \textbf{67.25} & 92.62 & \textbf{95.94} & 74.29 & 78.14 \\
  \hline
\multirow{4}{*}{LMNN} & PD-cone & 63.10 & 82.21 & 77.47 & 62.95 & 95.61 & 97.00 & 83.68 & 79.87 \\
 & HBNB & \textbf{63.67} & \textbf{84.90} & \textbf{77.73} & 64.06 & \textbf{94.92} & 95.94 & 75.65 & 79.34 \\
 & PGML & 61.67 & 79.87 & 76.18 & 62.37 & 93.30 & 95.94 & 77.01 & 78.77 \\
 & SGML & \textbf{63.67} & 78.25 & 76.70 & \textbf{64.59} & 92.83 & \textbf{96.65} & \textbf{77.51} & 79.02 \\
  \hline
\multirow{4}{*}{GLR} & PD-cone & 63.05 & 80.04 & 77.99 & 61.43 & 95.16 & 96.12 & 79.44 & 79.18 \\
 & HBNB & 63.71 & \textbf{81.49} & \textbf{76.17} & 60.85 & 95.61 & \textbf{96.30} & \textbf{76.58} & 79.08 \\
  & PGML & 63.62 & 80.58 & 74.88 & 62.37 & \textbf{96.77} & 95.94 & 73.77 & 78.69 \\
 & SGML & \textbf{64.38} & 80.04 & 74.48 & \textbf{67.31} & 96.09 & 95.41 & 75.15 & 79.38 \\
 \hline
\end{tabular} 
\end{scriptsize}
\end{center}
\end{table*}

\section{Conclusion}
\label{sec:conclude}
A fundamental challenge in metric learning is to efficiently handle the constraint of metric matrix $\M$ inside the positive definite (PD) cone. 
Circumventing full eigen-decomposition, we propose a fast, general optimization framework capable of minimizing any convex and differentiable objective $Q(\M)$.
The theoretical foundation is Gershgorin disc perfect alignment (GDPA): all Gershgorin disc left-ends of a generalized graph Laplacian matrix corresponding to an irreducible, balanced signed graph can be perfectly aligned via a similarity transform. 
This enables us to write tightest possible linear constraints per iteration to replace the PD cone constraint, and to solve the optimization as a sequence of linear programs via the Frank-Wolfe method.
We envision that GDPA can also be used in other optimization problems with PD / PSD cone constraints, such as semi-definite programs (SDP). 

\appendix
\section{Appendix}
\subsection*{Proof of Unit Spectral Radius}

We prove that matrix $\A = \D^{-1}(\W_g + \lambda_{\min} \I)$ defined in Lemma\;1 has unit spectral radius, \ie, $\rho(\A) = 1$.
To show this, suppose the contrary and $\A$ has eigenvalue $\theta > 1$, with corresponding unit-norm right eigenvector $\u$, \ie, 
\begin{align}
\theta \u &= \D^{-1} (\W_g + \lambda_{\min} \I) \u \nonumber \\
\left( \theta \D - \W_g \right) \u &=  \lambda_{\min} \u.
\nonumber
\end{align}
Thus, $\u$ is also an eigenvector for symmetric matrix $\theta \D - \W_g$ corresponding to eigenvalue $\lambda_{\min}$.
We can then write the Rayleigh quotient \cite{ma2012} for eigen-pair $(\lambda_{\min}, \u)$:
\begin{align}
\frac{\u^{\top} \left( \theta \D - \W_g \right) \u}{\u^{\top} \u} &= \lambda_{\min}
\nonumber \\
\theta \u^{\top} \D \u - \u^{\top} \W_g \u &= \lambda_{\min}
\nonumber
\end{align}
where $\u$ has unit-norm. 
Since $\theta > 1$ by assumption and $\D$ is PD, we can write
\begin{align}
\lambda_{\min} =
\theta \u^{\top} \D \u - \u^{\top} \W_g \u &> \u^{\top} \D \u - \u^{\top} \W_g \u
\nonumber \\
&= \frac{\u^{\top} \left( \D - \W_g \right) \u}{\u^{\top} \u}.
\nonumber
\end{align}
Thus $\u$ achieves a smaller Rayleigh quotient for matrix $\D - \W_g$ than $\lambda_{\min}$ using first eigenvector $\v$. 
This is a contradiction, and hence $\theta > 1$ does not exist.

Suppose now $\theta < -1$. Starting again from the Rayleigh quotient for eigen-pair $(\lambda_{\min}, \u)$ of matrix $\theta \D - \W_g$:
\begin{align}
0 &\leq \lambda_{\min} = \theta \u^{\top} \D \u - \u^{\top} \W_g \u
\nonumber \\
\u^{\top} \W_g \u &\leq \theta \u^{\top} \D \u 
~< \; -\u^{\top} \D \u
\nonumber
\end{align}
where the strict inequality is true since $\theta < -1$ by assumption. 
Define (self-loop-free) degree matrix $\D_g = \text{diag}(\W_g \1)$ and graph Laplacian $\L_g = \D_g - \W_g$ that correspond to adjacency matrix $\W_g$.
We can write 
\begin{align}
\u^{\top} \W_g \u &= \u^{\top} \D_g \u - \u^{\top} \L_g \u.
\nonumber 
\end{align}
The strict inequality can be rewritten as:
\begin{align}
\u^{\top} \L_g \u - \u^{\top} \D_g \u &> \u^{\top} \D \u
\nonumber \\
\u^{\top} \L_g \u &> \u^{\top} (\D + \D_g) \u 
\geq 2 \u^{\top} \D_g \u 
\nonumber \\
\sum_{(i,j) \in \cE} w_{ij} (u_i - u_j)^2 &> 2 \sum_i \left( \sum_j w_{ij} \right) u_i^2 
\nonumber \\
&= 2 \sum_{(i,j) \in \cE} w_{ij} (u_i^2 + u_j^2)
\nonumber \\
\sum_{(i,j)} w_{ij} (-2 u_i u_j) &> \sum_{(i,j) \in \cE} w_{ij} (u_i^2 + u_j^2).
\nonumber 
\end{align}
From $0 \leq (u_i + u_j)^2$, we know $-2 u_i u_j \leq u_i^2 + u_j^2$.
Thus the last equation is a contradiction, and $\theta < -1$ is also not possible.
Thus $\rho(\A) = 1$. $\Box$


\bibliographystyle{IEEEtran}
\bibliography{07_ref}

\begin{thebibliography}{10}
\providecommand{\url}[1]{#1}
\csname url@samestyle\endcsname
\providecommand{\newblock}{\relax}
\providecommand{\bibinfo}[2]{#2}
\providecommand{\BIBentrySTDinterwordspacing}{\spaceskip=0pt\relax}
\providecommand{\BIBentryALTinterwordstretchfactor}{4}
\providecommand{\BIBentryALTinterwordspacing}{\spaceskip=\fontdimen2\font plus
\BIBentryALTinterwordstretchfactor\fontdimen3\font minus
  \fontdimen4\font\relax}
\providecommand{\BIBforeignlanguage}[2]{{%
\expandafter\ifx\csname l@#1\endcsname\relax
\typeout{** WARNING: IEEEtran.bst: No hyphenation pattern has been}%
\typeout{** loaded for the language `#1'. Using the pattern for}%
\typeout{** the default language instead.}%
\else
\language=\csname l@#1\endcsname
\fi
#2}}
\providecommand{\BIBdecl}{\relax}
\BIBdecl

\bibitem{weinNNberger09LMNN}
K.~Q. Weinberger and L.~K. Saul, ``Distance metric learning for large margin
  nearest neighbor classification,'' \emph{Journal of Machine Learning
  Research}, vol.~10, no.~2, pp. 207--244, Feb. 2009.

\bibitem{mahalanobis1936}
P.~C. Mahalanobis, ``On the generalized distance in statistics,''
  \emph{Proceedings of the National Institute of Sciences of India}, vol.~2,
  no.~1, pp. 49--55, April 1936.

\bibitem{fsp2014}
M.~Vetterli, J.~Kovacevic, and V.~Goyal, \emph{Foundations of Signal
  Processing}.\hskip 1em plus 0.5em minus 0.4em\relax Cambridge University
  Press, 2014.

\bibitem{schultz04nips}
M.~Schultz and T.~Joachims, ``Learning a distance metric from relative
  comparisons,'' in \emph{Annual Conference on Neural Information Processing
  Systems}, 2004, pp. 41--48.

\bibitem{globerson2006metric}
A.~Globerson and S.~T. Roweis, ``Metric learning by collapsing classes,'' in
  \emph{Annual Conference on Neural Information Processing Systems}, 2006, pp.
  451--458.

\bibitem{qi09icml}
G.-J. Qi, J.~Tang, Z.-J. Zha, T.-S. Chua, and H.-J. Zhang, ``An efficient
  sparse metric learning in high-dimensional space via {$l_1$}-penalized
  log-determinant regularization,'' in \emph{International Conference on
  Machine Learning}, June 2009, pp. 841--848.

\bibitem{zadeh16GMML}
P.~Zadeh, R.~Hosseini, and S.~Sra, ``Geometric mean metric learning,'' in
  \emph{International Conference on Machine Learning}, June 2016, pp.
  2464--2471.

\bibitem{Parikh31}
N.~Parikh and S.~Boyd, ``Proximal algorithms,'' \emph{Foundations and Trends in
  Optimization}, vol.~1, no.~3, pp. 127--239, Jan. 2014.

\bibitem{GoluVanl96}
G.~H. Golub and C.~F. Van~Loan, \emph{Matrix Computations}, 3rd~ed.\hskip 1em
  plus 0.5em minus 0.4em\relax The Johns Hopkins University Press, 1996.

\bibitem{yang20}
C.~Yang, G.~Cheung, and W.~Hu, ``Graph metric learning via {Gershgorin} disc
  alignment,'' in \emph{IEEE International Conference on Acoustics, Speech and
  Signal Processing}, May 2020.

\bibitem{hu2020feature}
W.~Hu, X.~Gao, G.~Cheung, and Z.~Guo, ``Feature graph learning for 3{D} point
  cloud denoising,'' \emph{IEEE Transactions on Signal Processing}, vol.~68,
  pp. 2841--2856, 2020.

\bibitem{LSML}
E.~Y. {Liu}, Z.~{Guo}, X.~{Zhang}, V.~{Jojic}, and W.~{Wang}, ``Metric learning
  from relative comparisons by minimizing squared residual,'' in \emph{IEEE
  International Conference on Data Mining}, Dec. 2012, pp. 978--983.

\bibitem{ericdml}
E.~P. Xing, M.~I. Jordan, S.~J. Russell, and A.~Y. Ng, ``Distance metric
  learning with application to clustering with side-information,'' in
  \emph{Annual Conference on Neural Information Processing Systems}, Dec. 2003,
  pp. 521--528.

\bibitem{lim13icml}
D.~Lim, G.~Lanckriet, and B.~McFee, ``Robust structural metric learning,'' in
  \emph{International Conference on Machine Learning}, June 2013, pp. 615--623.

\bibitem{liu15aaai}
W.~Liu, C.~Mu, R.~Ji, S.~Ma, J.~R. Smith, and S.-F. Chang, ``Low-rank
  similarity metric learning in high dimensions,'' in \emph{AAAI Conference on
  Artificial Intelligence}, Jan. 2015, p. 2792–2799.

\bibitem{mu16aaai}
Y.~Mu, ``Fixed-rank supervised metric learning on {R}iemannian manifold,'' in
  \emph{AAAI Conference on Artificial Intelligence}, Feb. 2016, pp. 1941--1947.

\bibitem{zhang17aaai}
J.~Zhang and L.~Zhang, ``Efficient stochastic optimization for low-rank
  distance metric learning,'' in \emph{AAAI Conference on Artificial
  Intelligence}, Feb. 2017, pp. 933--939.

\bibitem{yang2018apsipa}
C.~Yang, G.~Cheung, and V.~Stankovic, ``Alternating binary classifier and graph
  learning from partial labels,'' in \emph{Asia Pacific Signal and Information
  Processing Association Annual Summit and Conference}, Nov. 2018, pp.
  1137--1140.

\bibitem{biyikoglu2005nodal}
T.~Biyikoglu, J.~Leydold, and P.~F. Stadler, ``Nodal domain theorems and
  bipartite subgraphs,'' \emph{The {E}lectronic {J}ournal of {L}inear
  {A}lgebra}, vol.~13, pp. 344--351, Jan. 2005.

\bibitem{gahc}
R.~S. Varga, \emph{{G}ershgorin and his circles}.\hskip 1em plus 0.5em minus
  0.4em\relax Springer, 2004.

\bibitem{co1998}
C.~Papadimitriou and K.~Steiglitz, \emph{Combinatorial Optimization}.\hskip 1em
  plus 0.5em minus 0.4em\relax Dover Publications, Inc, 1998.

\bibitem{pmlr-v28-jaggi13}
M.~Jaggi, ``Revisiting {Frank-Wolfe}: Projection-free sparse convex
  optimization,'' in \emph{International Conference on Machine Learning}, Jun
  2013, pp. 427--435.

\bibitem{Knyazev01}
A.~V. Knyazev, ``Toward the optimal preconditioned eigensolver: Locally optimal
  block preconditioned conjugate gradient method,'' \emph{SIAM Journal on
  Scientific Computing}, vol.~23, no.~2, pp. 517--541, 2001.

\bibitem{pang2017graph}
J.~Pang and G.~Cheung, ``Graph {L}aplacian regularization for image denoising:
  Analysis in the continuous domain,'' \emph{IEEE Transactions on Image
  Processing}, vol.~26, no.~4, pp. 1770--1785, April 2017.

\bibitem{torresani2007large}
L.~Torresani and K.-c. Lee, ``Large margin component analysis,'' in
  \emph{Annual Conference on Neural Information Processing Systems}, 2007, pp.
  1385--1392.

\bibitem{mika1999fisher}
S.~Mika, G.~Ratsch, J.~Weston, B.~Scholkopf, and K.-R. Mullers, ``Fisher
  discriminant analysis with kernels,'' in \emph{Neural networks for signal
  processing IX: Proceedings of the 1999 IEEE signal processing society
  workshop (cat. no. 98th8468)}, 1999, pp. 41--48.

\bibitem{lu2017deep}
J.~Lu, J.~Hu, and J.~Zhou, ``Deep metric learning for visual understanding: An
  overview of recent advances,'' \emph{IEEE Signal Processing Magazine},
  vol.~34, no.~6, pp. 76--84, 2017.

\bibitem{hadsell2006dimensionality}
R.~Hadsell, S.~Chopra, and Y.~LeCun, ``Dimensionality reduction by learning an
  invariant mapping,'' in \emph{IEEE Conference on Computer Vision and Pattern
  Recognition}, vol.~2, 2006, pp. 1735--1742.

\bibitem{taigman2014deepface}
Y.~Taigman, M.~Yang, M.~Ranzato, and L.~Wolf, ``Deepface: Closing the gap to
  human-level performance in face verification,'' in \emph{IEEE Conference on
  Computer Vision and Pattern Recognition}, 2014, pp. 1701--1708.

\bibitem{hu2014discriminative}
J.~Hu, J.~Lu, and Y.-P. Tan, ``Discriminative deep metric learning for face
  verification in the wild,'' in \emph{IEEE Conference on Computer Vision and
  Pattern Recognition}, 2014, pp. 1875--1882.

\bibitem{sun2014deep}
Y.~Sun, Y.~Chen, X.~Wang, and X.~Tang, ``Deep learning face representation by
  joint identification-verification,'' in \emph{Advances in neural information
  processing systems}, 2014, pp. 1988--1996.

\bibitem{wang2014learning}
J.~Wang, Y.~Song, T.~Leung, C.~Rosenberg, J.~Wang, J.~Philbin, B.~Chen, and
  Y.~Wu, ``Learning fine-grained image similarity with deep ranking,'' in
  \emph{IEEE Conference on Computer Vision and Pattern Recognition}, 2014, pp.
  1386--1393.

\bibitem{hoffer2015deep}
E.~Hoffer and N.~Ailon, ``Deep metric learning using triplet network,'' in
  \emph{International Workshop on Similarity-Based Pattern Recognition}.\hskip
  1em plus 0.5em minus 0.4em\relax Springer, 2015, pp. 84--92.

\bibitem{schroff2015facenet}
F.~Schroff, D.~Kalenichenko, and J.~Philbin, ``Facenet: A unified embedding for
  face recognition and clustering,'' in \emph{IEEE Conference on Computer
  Vision and Pattern Recognition}, 2015, pp. 815--823.

\bibitem{oh2016deep}
H.~Oh~Song, Y.~Xiang, S.~Jegelka, and S.~Savarese, ``Deep metric learning via
  lifted structured feature embedding,'' in \emph{IEEE Conference on Computer
  Vision and Pattern Recognition}, 2016, pp. 4004--4012.

\bibitem{ying2012distance}
Y.~Ying and P.~Li, ``Distance metric learning with eigenvalue optimization,''
  \emph{JMLR}, vol.~13, no. Jan, pp. 1--26, 2012.

\bibitem{bai19icassp}
Y.~Bai, G.~Cheung, F.~Wang, X.~Liu, and W.~Gao, ``Reconstruction-cognizant
  graph sampling using {G}ershgorin disc alignment,'' in \emph{International
  Conference on Acoustics, Speech and Signal Processing}, May 2019, pp.
  5396--5400.

\bibitem{bai20tsp}
Y.~Bai, F.~Wang, G.~Cheung, Y.~Nakatsukasa, and W.~Gao, ``Fast graph sampling
  set selection using {G}ershgorin disc alignment,'' \emph{IEEE Transactions on
  Signal Processing}, March 2020.

\bibitem{8296568}
W.~{Su}, G.~{Cheung}, and C.~{Lin}, ``Graph {Fourier} transform with negative
  edges for depth image coding,'' in \emph{IEEE International Conference on
  Image Processing}, 2017, pp. 1682--1686.

\bibitem{8347162}
A.~{Ortega}, P.~{Frossard}, J.~{Kovacevic}, J.~M.~F. {Moura}, and
  P.~{Vandergheynst}, ``Graph signal processing: Overview, challenges, and
  applications,'' \emph{Proceedings of the IEEE}, vol. 106, no.~5, pp.
  808--828, May 2018.

\bibitem{8334407}
G.~{Cheung}, E.~{Magli}, Y.~{Tanaka}, and M.~K. {Ng}, ``Graph spectral image
  processing,'' \emph{Proceedings of the IEEE}, vol. 106, no.~5, pp. 907--930,
  May 2018.

\bibitem{ortega_2021}
A.~Ortega, \emph{Introduction to Graph Signal Processing}.\hskip 1em plus 0.5em
  minus 0.4em\relax Cambridge University Press, 2021.

\bibitem{cheung_2021}
G.~Cheung and E.~Magli, Eds., \emph{Graph Spectral Image Processing}.\hskip 1em
  plus 0.5em minus 0.4em\relax ISTE/Wiley, 2021.

\bibitem{irregraph}
M.~Milgram, ``Irreducible graphs,'' \emph{Journal Of Combinatorial Theory (B)},
  vol.~12, pp. 6--31, Feb. 1972.

\bibitem{cheung2018graph}
G.~Cheung, E.~Magli, Y.~Tanaka, and M.~K. Ng, ``Graph spectral image
  processing,'' \emph{Proceedings of the IEEE}, vol. 106, no.~5, pp. 907--930,
  2018.

\bibitem{ma2012}
R.~Horn and C.~Johnson, \emph{Matrix Analysis}.\hskip 1em plus 0.5em minus
  0.4em\relax Cambridge University Press, 2012.

\bibitem{cht1956}
D.~{Cartwright} and F.~{Harary}, ``Structural balance: a generalization of
  heider's theory,'' \emph{Psychological Review}, vol.~63, no.~5, pp. 277--293,
  1956.

\bibitem{sgsn}
J.~Leskovec, D.~Huttenlocher, and J.~Kleinberg, ``Signed networks in social
  media,'' in \emph{SIGCHI Conference on Human Factors in Computing Systems},
  April 2010, p. 1361–1370.

\bibitem{classificationpami19}
M.~{Dong}, Y.~{Wang}, X.~{Yang}, and J.~{Xue}, ``Learning local metrics and
  influential regions for classification,'' \emph{IEEE Transactions on Pattern
  Analysis and Machine Intelligence}, vol.~42, no.~6, pp. 1522--1529, June
  2020.

\bibitem{co2009}
S.~Boyd and L.~Vandenberghe, \emph{Convex Optimization}.\hskip 1em plus 0.5em
  minus 0.4em\relax Cambridge University Press, 2009.

\bibitem{nrfwsso}
J.~{Raphson}, \emph{Analysis aequationum universalis}, London, 1690.

\bibitem{friedman2008sparse}
J.~Friedman, T.~Hastie, and R.~Tibshirani, ``Sparse inverse covariance
  estimation with the graphical lasso,'' \emph{Biostatistics}, vol.~9, no.~3,
  pp. 432--441, 2008.

\bibitem{10.1145/2184319.2184343}
E.~Cand\`{e}s and B.~Recht, ``Exact matrix completion via convex
  optimization,'' \emph{Commun. ACM}, vol.~55, no.~6, p. 111–119, Jun. 2012.

\bibitem{Wright-Ma-2021}
J.~Wright and Y.~Ma, \emph{High-Dimensional Data Analysis with Low-Dimensional
  Models: Principles, Computation, and Applications}.\hskip 1em plus 0.5em
  minus 0.4em\relax Cambridge University Press, 2021.

\bibitem{Bertsekas/99}
D.~Bertsekas, \emph{Nonlinear Programming: 3rd Edition}.\hskip 1em plus 0.5em
  minus 0.4em\relax Athena Scientific, 2016.

\bibitem{10.5555/1671238}
S.~Russell and P.~Norvig, \emph{Artificial Intelligence: A Modern Approach},
  3rd~ed.\hskip 1em plus 0.5em minus 0.4em\relax USA: Prentice Hall Press,
  2009.

\end{thebibliography}

\begin{IEEEbiography}[{\includegraphics[width=1in,height=1.25in,trim={0.5in 0in 0.5in 0in},clip,keepaspectratio]{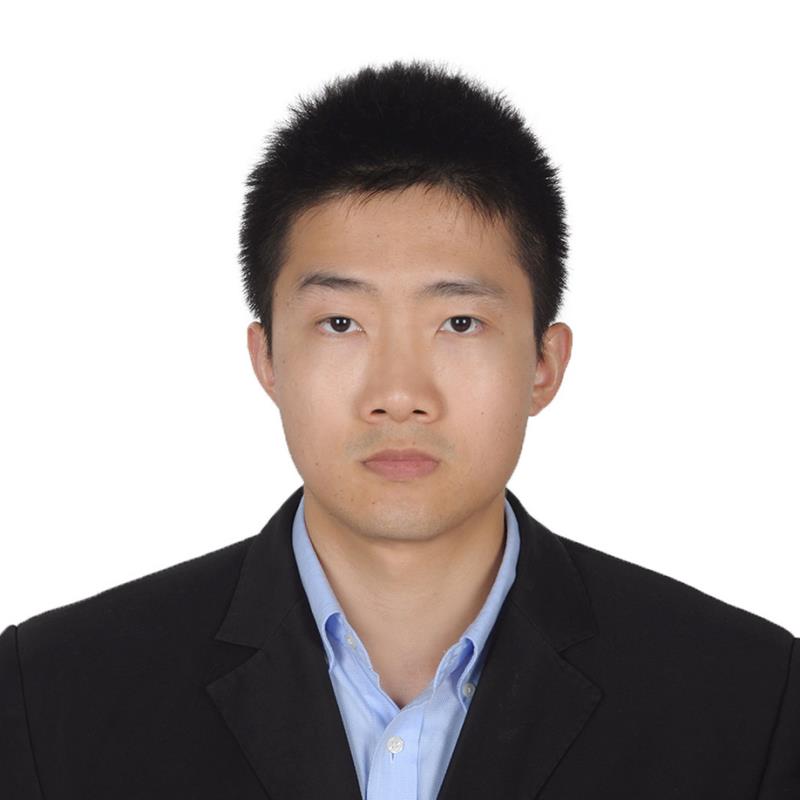}}]{Cheng Yang}
(A’11–S’12–M’14) received the B.Eng. and Ph.D. degrees in electronic and electrical engineering from the University of Strathclyde, Glasgow, U.K., in 2011 and 2017, respectively. 
He is a Postdoc at Shanghai Jiao Tong University, Shanghai, China.
He was a Postdoc at York University, Toronto,
Canada 2019-2020 and a Project Researcher at National
Institute of Informatics, Tokyo, Japan 2017-2018. 
His research interests include graph signal processing and multimedia systems.
\end{IEEEbiography}

\begin{IEEEbiography}[{\includegraphics[width=1in,height=1.25in,trim={0in 0in 0in 0in},clip,keepaspectratio]{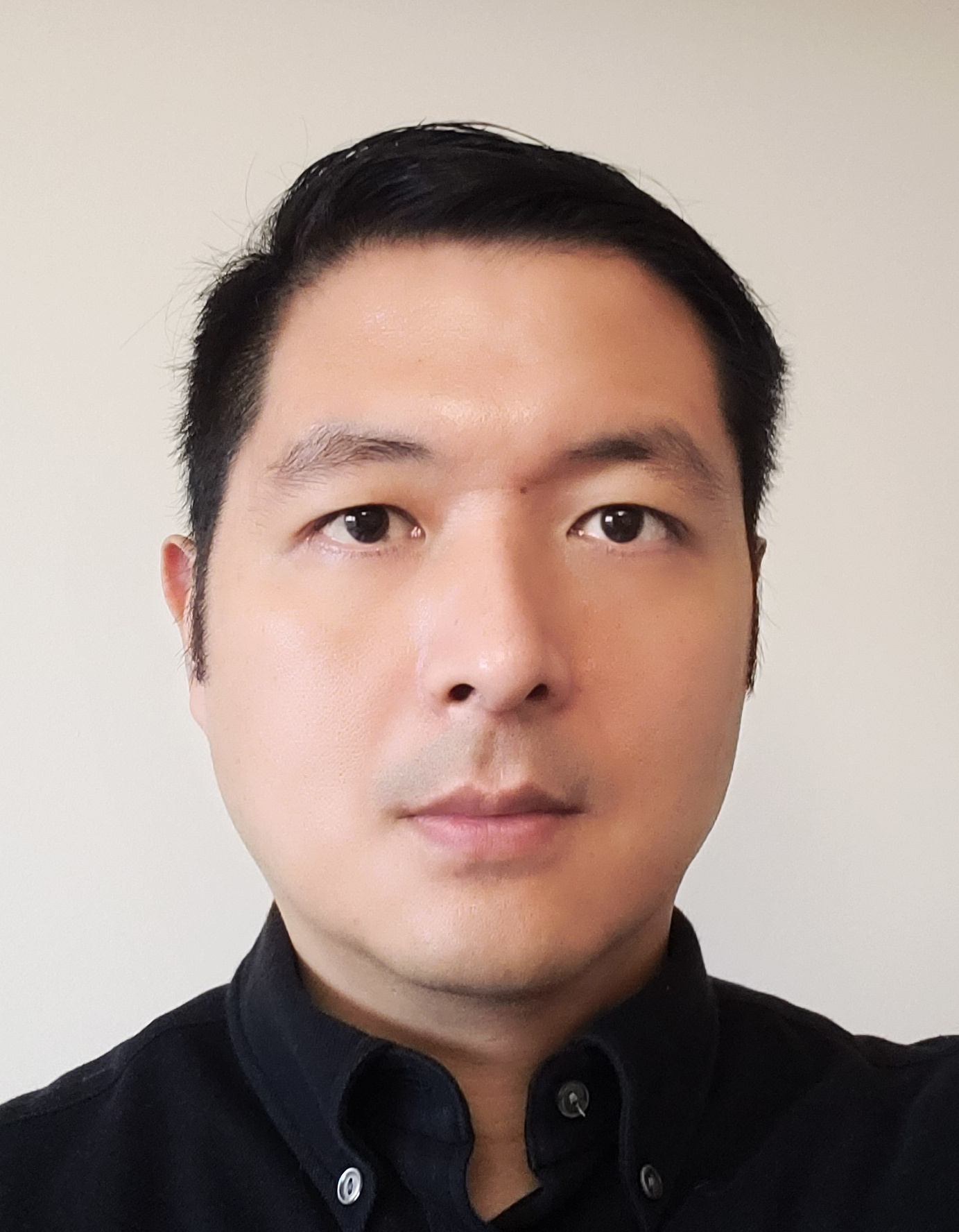}}]{Gene Cheung} (M'00--SM'07--F'21)
received the B.S. degree in electrical engineering from Cornell University in 1995, and the M.S. and Ph.D. degrees in electrical engineering and computer science from the University of California, Berkeley, in 1998 and 2000, respectively. 

He was a senior researcher in Hewlett-Packard Laboratories Japan, Tokyo, from 2000 till 2009. 
He was an assistant then associate professor in National Institute of Informatics (NII) in Tokyo, Japan, from 2009 till 2018. 
He is now an associate professor in York University, Toronto, Canada.

His research interests include 3D imaging and graph signal processing. 
He has served as associate editor for multiple journals, including IEEE Transactions on Multimedia (2007--2011), IEEE Transactions on Circuits and Systems for Video Technology (2016--2017) and IEEE Transactions on Image Processing (2015--2019).
He is a co-author of several paper awards, including the best student paper award in ICIP 2013, ICIP 2017 and IVMSP 2016, best paper runner-up award in ICME 2012, and IEEE Signal Processing Society (SPS) Japan best paper award 2016.
He is a recipient of the Canadian NSERC Discovery Accelerator Supplement (DAS) 2019.
He is a fellow of IEEE.
\end{IEEEbiography}

\begin{IEEEbiography}[{\includegraphics[width=1in,height=1.25in,trim={0in 0in 0in 0in},clip,keepaspectratio]{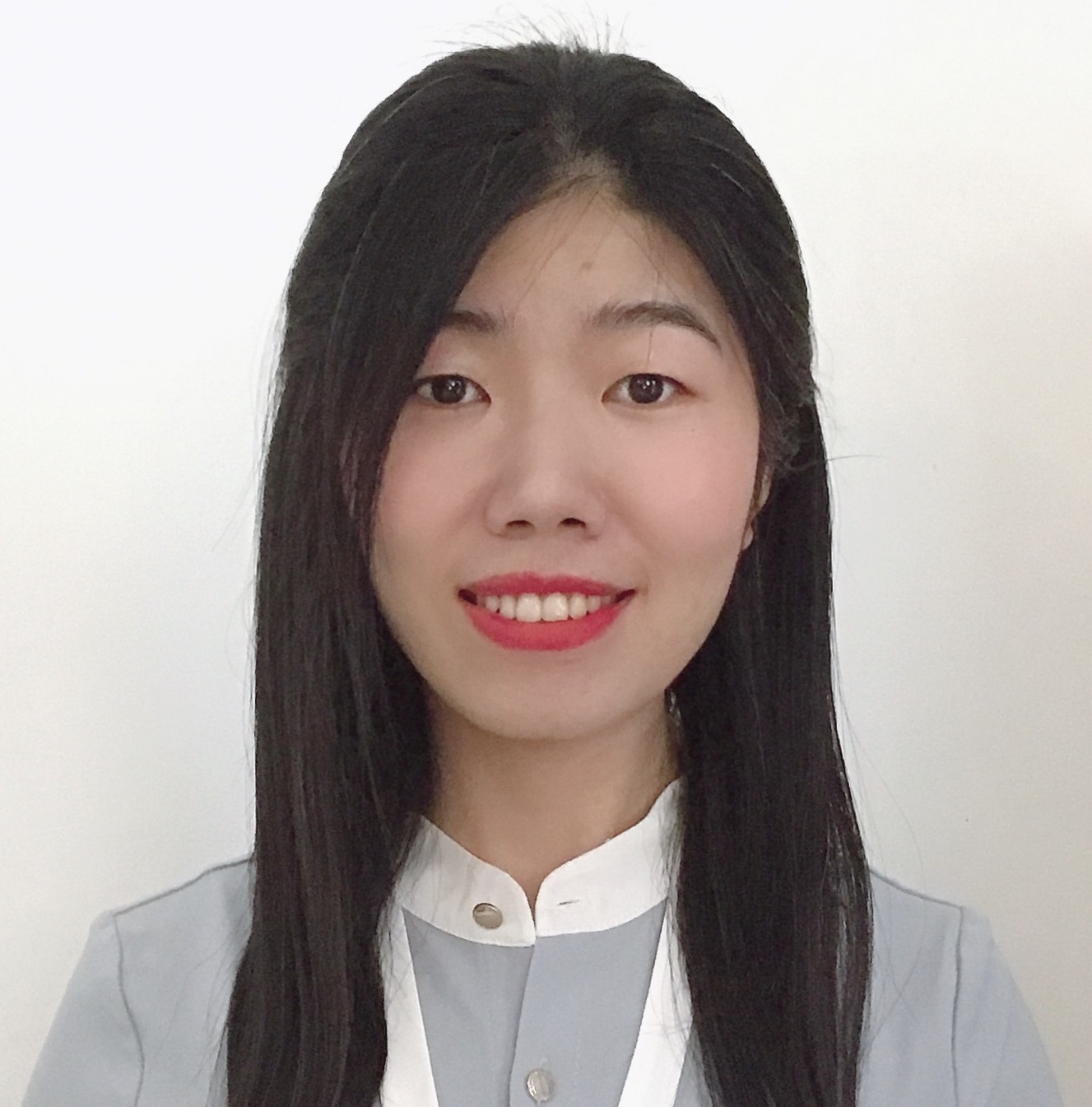}}]{Wei Hu}
(Senior Member, IEEE) received the B.S. degree in Electrical Engineering from the University of Science and Technology of China in 2010, and the Ph.D. degree in Electronic and Computer Engineering from the Hong Kong University of Science and Technology in 2015.
She was a Researcher with Technicolor, Rennes, France, from 2015 to 2017. She is currently an Assistant Professor with Wangxuan Institute of Computer Technology, Peking University. Her research interests are graph signal processing, graph-based machine learning and 3D visual computing. She has authored around 50 international journal and conference publications, with several paper awards including the best student paper runner up award in ICME 2020. She is a member in IEEE MSA-TC (2020-2024), and serves as associate editor for IEEE Transactions on Signal and Information Processing over Networks and Frontiers in Signal Processing. 
\end{IEEEbiography}


\end{document}